\documentclass{article}

\usepackage{times}
\usepackage{graphicx} 
\usepackage{subfigure} 

\usepackage{natbib}

\usepackage{algorithm}
\usepackage{algorithmic}

\usepackage{hyperref}

\usepackage[accepted]{icml2016}

\usepackage{amsmath}
\usepackage{amsthm}
\usepackage{amssymb}
\usepackage{booktabs}
\usepackage{color}
\usepackage{mathabx}
\usepackage{wrapfig}
\usepackage{enumitem}
\renewcommand{\algorithmicrequire}{\textbf{Input:}}
\renewcommand{\algorithmicensure}{\textbf{Output:}}

\DeclareMathOperator*{\argmin}{argmin}
\newcommand{\abs}[1]{\lvert#1\rvert}
\newcommand{\norm}[1]{\left\|#1\right\|}
\newcommand{\inner}[1]{\langle#1\rangle}

\usepackage{forloop}
\newcounter{ct}

\newcommand{\ah}{\widehat{a}}

\newcommand{\pis}{\pi^*}
\newcommand{\pih}{\hat{\pi}}
\newcommand{\pip}{\pi^\prime}

\newcommand{\E}{\mathbb{E}}

\newcommand{\X}{\textbf{X}}

\newcommand{\s}{\textbf{S}}
\newcommand{\x}{\textbf{x}}

\newcommand{\node}{\texttt{node}}
\newcommand{\mean}{\texttt{mean}}
\newcommand{\A}{\textbf{A}}

\newcommand{\train}{\texttt{Train}}
\newcommand{\regress}{\texttt{Regress}}

\graphicspath{{figures/}}

\theoremstyle{plain}
\newtheorem{thm}{Theorem}[section]
\newtheorem{lem}[thm]{Lemma}
\newtheorem{prop}[thm]{Proposition}
\newtheorem{cor}[thm]{Corollary}

\newtheorem{cond}{Condition}
\newtheorem*{lemma_statement}{\textit{Lemma Statement}}
\newtheorem*{theorem_statement}{\textit{Theorem Statement}}
\newtheorem*{prop_statement}{\textit{Proposition Statement}}

\theoremstyle{definition}
\newtheorem{defn}{Definition}[section]

\newtheorem{exmp}{Example}[section]

\theoremstyle{remark}
\newtheorem*{rem}{Remark}

\icmltitlerunning{Smooth Imitation Learning for Online Sequence Prediction}

\begin{document} 
\setlength{\abovedisplayskip}{2pt}
\setlength{\belowdisplayskip}{2pt}
\setlength{\abovedisplayshortskip}{3pt}
\setlength{\belowdisplayshortskip}{3pt}

\twocolumn[
\icmltitle{Smooth Imitation Learning for Online Sequence Prediction}

\icmlauthor{Hoang M. Le}{hmle@caltech.edu}
\icmlauthor{Andrew Kang}{akang@caltech.edu}
\icmlauthor{Yisong Yue}{yyue@caltech.edu}
\icmladdress{California Institute of Technology, Pasadena, CA, USA}
\icmlauthor{Peter Carr}{peter.carr@disneyresearch.com}
\icmladdress{Disney Research, Pittsburgh, PA, USA}
\icmlkeywords{imitation learning, learning reduction, approximate policy iteration}

\vskip 0.3in
]

\begin{abstract} 
We study the problem of smooth imitation learning for online sequence prediction, where the goal is to train a policy that can smoothly imitate demonstrated behavior in a dynamic and continuous environment in response to online, sequential context input. Since the mapping from context to behavior is often complex, we take a learning reduction approach to reduce smooth imitation learning to a regression problem using complex function classes that are regularized to ensure smoothness. We present a learning meta-algorithm that achieves fast and stable convergence to a good policy. Our approach enjoys several attractive properties, including being fully deterministic, employing an adaptive learning rate that can provably yield larger policy improvements compared to previous approaches, and the ability to ensure stable convergence. Our empirical results demonstrate significant performance gains over previous approaches.
\end{abstract} 

\vspace{-0.2in}
\section{Introduction}
\vspace{-0.5em}
In many complex planning and control tasks, it can  be very challenging to explicitly specify a good policy.  For such tasks, the use of machine learning to automatically learn a good policy from observed expert behavior, also known as imitation learning or learning from demonstrations, has proven tremendously useful \cite{abbeel2004apprenticeship,ratliff2009learning,argall2009survey,ross2010efficient,dagger,jain2013learning}.

In this paper, we study the problem of imitation learning for smooth online sequence prediction in a continuous regime. Online sequence prediction is the problem of making  online decisions in response to exogenous input from the environment, and is a special case of reinforcement learning (see Section \ref{sec:problem}).  We are further interested in policies that make smooth predictions in a continuous action space.

Our motivating example is the problem of learning smooth policies for automated camera planning \cite{smooth_camera}: 
determining where a camera should look given environment information (e.g., noisy person detections) and corresponding demonstrations from a human expert.\footnote{Access data at \url{http://www.disneyresearch.com/publication/smooth-imitation-learning/} and code at \url{http://github.com/hoangminhle/SIMILE}.} 
It is widely accepted that a smoothly moving camera is essential for generating aesthetic video \cite{gaddam2015cameraman}.
From a problem formulation standpoint, one key difference between smooth imitation learning and conventional imitation learning is the use of a ``smooth'' policy class (which we formalize in Section \ref{sec:problem}), and the goal now is to mimic expert demonstrations by choosing the best smooth policy. 

The conventional supervised learning approach to imitation learning is to train a classifier or regressor to predict the expert's behavior given training data comprising input/output pairs of contexts and actions taken by the expert. However, the learned policy's prediction affects (the distribution of) future states during the policy's actual execution, and so violates the crucial i.i.d. assumption made by most statistical learning approaches. To address this issue, numerous learning reduction approaches have been proposed \cite{daume2009search,ross2010efficient,dagger}, which iteratively modify the training distribution in various ways such that any supervised learning guarantees provably lift to the sequential imitation setting (potentially at the cost of statistical or computational efficiency).

We present a learning reduction approach to smooth imitation learning for online sequence prediction, which we call SIMILE (\textbf{S}mooth \textbf{IMI}tation \textbf{LE}arning). Building upon learning reductions that employ policy aggregation \cite{daume2009search}, we provably lift supervised learning guarantees to the smooth imitation setting and show much faster convergence behavior compared to previous work.
Our contributions can be summarized as:
\vspace{-1.2em}
\begin{itemize}
	\itemsep-0.2em
\item We formalize the problem of smooth imitation learning for online sequence prediction, and introduce a family of smooth policy classes  that is amenable to supervised learning reductions.
\item We present a principled learning reduction approach, which we call SIMILE.
Our approach enjoys several attractive practical properties, including learning a fully deterministic stationary policy (as opposed to SEARN \cite{daume2009search}), and not requiring data aggregation (as opposed to DAgger \cite{dagger}) which can lead to super-linear training time.
\item We provide performance guarantees that lift the the underlying supervised learning guarantees to the smooth imitation setting.  Our guarantees hold in the agnostic setting, i.e., when the supervised learner  might not achieve perfect prediction.
\item We show how to exploit a stability property of our smooth policy class to enable adaptive learning rates that yield provably much faster convergence compared to SEARN \cite{daume2009search}.
\item We empirically evaluate using the setting of smooth camera planning \cite{smooth_camera}, and demonstrate the performance gains of our approach.
\end{itemize}

\vspace{-0.12in}
\section{Problem Formulation}
\label{sec:problem}
\vspace{-0.2em}
Let  $\X = \{x_1,\ldots,x_T\} \subset \mathcal{X}^T$ denote a context sequence from the environment $\mathcal{X}$, and $\A = \{a_1,\ldots,a_T\}\subset\mathcal{A}^T$ denote an action sequence from some action space $\mathcal{A}$.  Context sequence is exogenous, meaning $a_t$ does not influence future context $x_{t+k}$ for $k\geq 1$. Let $\Pi$ denote a policy class, where each $\pi\in\Pi$ generates an action sequence $\A$ in response to a context sequence $\X$.
Assume $\mathcal{X}\subset\mathbb{R}^m,\mathcal{A}\subset\mathbb{R}^k$ are continuous and infinite, with $\mathcal{A}$ non-negative and bounded such that $\boldsymbol{\vec{0}} \preceq a \preceq R\boldsymbol{\vec{1}} \enskip \forall a \in \mathcal{A}$. 

Predicting actions $a_t$ may depend on recent contexts $x_t,\ldots,x_{t-p}$ and actions $a_{t-1}, \ldots, a_{t-q}$. 
Without loss of generality, we define a state space $\mathcal{S}$ as $\{s_t = \left[ x_t,a_{t-1}\right]\}$.\footnote{We can always concatenate consecutive contexts and actions.} 
Policies $\pi$ can thus be viewed as mapping states $\mathcal{S} = \mathcal{X}\times\mathcal{A}$ to actions $\mathcal{A}$.
A roll-out of $\pi$ given context sequence $\X = \{x_1,\ldots,x_T\}$ is the action sequence $\A = \{a_1,\ldots,a_T\}$:
\begin{align*}
a_t &= \pi(s_t)=\pi(\left[x_t,a_{t-1}\right]), \\ 
s_{t+1} &= \left[ x_{t+1},a_t\right] \quad \forall t\in\left[1,\ldots,T\right].
\end{align*}
Note that unlike the general reinforcement learning problem, we consider the setting where the state space splits into external and internal components (by definition, $a_t$ influences subsequent states $s_{t+k}$, but not $x_{t+k}$). The use of exogenous contexts $\{x_t\}$ models settings where a policy needs to take online, sequential actions based on external environmental inputs, e.g. smooth self-driving vehicles for obstacle avoidance, helicopter aerobatics in the presence of turbulence, or smart grid management for external energy demand.  The technical motivation of this dichotomy is that we will enforce smoothness only on the internal state.




Consider the example of autonomous camera planning for broadcasting a sport event \cite{smooth_camera}. $\mathcal{X}$ can correspond to game information such as the locations of the players, the ball, etc., and  $\mathcal{A}$ can correspond to the pan-tilt-zoom configuration of the broadcast camera.  Manually specifying a good camera policy can be very challenging due to sheer complexity involved with mapping $\mathcal{X}$ to $\mathcal{A}$. It is much more natural to train $\pi\in\Pi$ to mimic observed expert demonstrations.  For instance, $\Pi$ can be the space of neural networks or tree-based ensembles (or both).



Following the basic setup from \cite{dagger}, for any policy $\pi\in\Pi$, let $d_t^\pi$ denote the distribution of states at time $t$ if $\pi$ is executed  for the first $t-1$ time steps. Furthermore, let $d_\pi = \frac{1}{T}\sum_{t=1}^T d_t^\pi$ be the average distribution of states if we follow $\pi$ for all $T$ steps.
The goal of imitation learning is to find a policy $\pih\in\Pi$ which minimizes the imitation loss under its own induced distribution of states: 
\begin{eqnarray} 
\pih = \argmin_{\pi\in\Pi}\ell_{\pi}(\pi) = \argmin_{\pi\in\Pi}\E_{s \sim d_{\pi}}\left[ \ell(\pi(s)) \right],\label{eqn:obj}
\end{eqnarray}
where the (convex) imitation loss $\ell(\pi(s))$ captures how well $\pi$ imitates expert demonstrations for state $s$.
One common $\ell$ is squared loss between the policy's decision and the expert demonstration: $\ell(\pi(s)) = \|\pi(s)-\pis(s)\|^2$ for some norm $\norm{.}$.
Note that computing $\ell$ typically requires having access to a training set of expert demonstrations $\pis$ on some set of context sequences. 
We also assume an agnostic setting, where the minimizer of \eqref{eqn:obj} does not necessarily achieve 0 loss (i.e. it cannot perfectly imitate the expert). 
\vspace{-0.05in}
\subsection{Smooth Imitation Learning \& Smooth Policy Class}
In addition to accuracy, 
   a key requirement of many continuous control and planning problems is smoothness (e.g., smooth camera trajectories). Generally, ``smoothness" may reflect domain knowledge about stability properties or approximate equilibria of a dynamical system. 
   We thus formalize the problem of \textit{smooth imitation learning} as minimizing \eqref{eqn:obj} over a smooth policy class $\Pi$. 
   
   Most previous work on learning smooth policies focused on simple policy classes such as linear models \cite{abbeel2004apprenticeship}, which can be overly restrictive. We instead define a much more general smooth policy class $\Pi$ as a regularized space of complex models.
\vspace{-0.5em}
\begin{defn}[Smooth policy class $\Pi$]
	\label{def:policy_class}
	\textit{ Given a complex model class $\mathcal{F}$ and a class of smooth regularizers $\mathcal{H}$, we define smooth policy class $\Pi\subset\mathcal{F}\times\mathcal{H}$ as satisfying:\\
		\begin{align*}
			\Pi \triangleq \{	\pi = (f,h), &f\in\mathcal{F}, h\in\mathcal{H}\enskip | \enskip \pi(s) \text{ is close to} \\ &\text{ both } f(x,a) \text{ and } h(a)  \\
			& \forall \text{ induced state } s=[x,a]\in\mathcal{S}  \}
		\end{align*}
		where closeness is controlled by regularization.}
\end{defn}
\vspace{-0.1in}
For instance, $\mathcal{F}$ can be the space of neural networks or decision trees and $\mathcal{H}$ be the space of  smooth analytic functions. 
$\Pi$ can thus be viewed as policies that predict close to some $f\in\mathcal{F}$ but are regularized to be close to some $h\in\mathcal{H}$.
For sufficiently expressive $\mathcal{F}$, we often have that $\Pi\subset \mathcal{F}$.  Thus optimizing over $\Pi$ can be viewed as constrained optimization over $\mathcal{F}$ (by $\mathcal{H}$), which can be challenging.
Our SIMILE approach integrates alternating optimization (between $\mathcal{F}$ and $\mathcal{H}$) into the learning reduction. 
We provide two concrete examples of $\Pi$ below.


\begin{exmp}[$\Pi_{\lambda}$]
	\label{example:smooth_policy_class}
   Let $\mathcal{F}$ be any complex supervised model class, and define the simplest possible $\mathcal{H} \triangleq \{ h(a) = a\}$. Given $f\in\mathcal{F}$, the prediction of a policy $\pi$  can be viewed as regularized optimization over the action space to ensure closeness of $\pi$ to both $f$ and $h$:
	\setlength{\abovedisplayskip}{3pt}
	\setlength{\belowdisplayskip}{3pt}
	\begin{align}
	\pi(x,a) &= \argmin_{a^\prime\in\mathcal{A}}\norm{f(x,a)-a^\prime}^2+\lambda\norm{h(a) -a^\prime}^2\nonumber\\
	&= \frac{f(x,a) + \lambda h(a)}{1+\lambda} = \frac{f(x,a) + \lambda a}{1+\lambda},\label{eqn:pi}
	\end{align}
	where regularization parameter $\lambda$ trades-off closeness to $f$ and to previous action. For large $\lambda$, $\pi(x,a)$ is encouraged make predictions that stays close to previous action $a$.
\end{exmp}

	

\begin{exmp}[Linear auto-regressor smooth regularizers]
	\label{example:autoregressor_policy_class}
Let $\mathcal{F}$ be any complex supervised model class, and define $\mathcal{H}$ using linear auto-regressors, $\mathcal{H}\triangleq\{h(a) = \theta^\top a\}$, which model actions as a linear dynamical system \cite{wold1939study}. We can define $\pi$ analogously to \eqref{eqn:pi}.
\end{exmp}


\vspace{-0.05in}
In general, SIMILE requires that $\Pi$ satisfies a smooth property stated below. 
This property, which is exploited in our theoretical analysis (see Section \ref{sec:theory}), is motivated by the observation that given a (near) constant stream of context sequence, a stable behavior policy should exhibit a corresponding action sequence with low curvature.  The two examples above satisfy this property for sufficiently large $\lambda$.
\begin{defn}[$H$-state-smooth imitation policy]
	\label{defn:smooth_policy}
For small constant $0< H \ll 1$, a policy $\pi(\left[x,a\right])$ is $H$-state-smooth if it is $H$-smooth w.r.t. $a$, i.e. for fixed $x\in\mathcal{X}$, $\forall a,a^\prime \in \mathcal{A}$, $\forall i$: $\norm{\nabla\pi^i([x,a]) - \nabla\pi^i([x,a^\prime])}_{*}\leq H\norm{a-a^\prime} \enskip$ where $\pi^i$ indicates the $i^{th}$ component of vector-valued function\footnote{This emphasizes the possibility that $\pi$ is a vector-valued function of $a$. The gradient and Hessian are viewed as arrays of $k$ gradient vectors and Hessian matrices of 1-d case, since we simply treat action in $\mathbb{R}^k$ as an array of $k$ standard functions.} $\pi(s) = \left[\pi^1(s),\ldots,\pi^k(s)\right]\in\mathbb{R}^k$, and $\norm{.}$ and $\norm{.}_{*}$ are some norm and dual norm respectively. For twice differentiable policy $\pi$, this is equivalent to having the bound on the Hessian $\nabla^2\pi^i([x,a])\preceq H\mathbb{I}_k \enskip \forall i$.
\end{defn}

\vspace{-0.1in}
\section{Related Work}
\label{sec:related}
The most popular traditional approaches for learning from expert demonstration focused on using approximate policy iteration techniques in the MDP setting \cite{kakade2002approximately, bagnell2003policy}. Most prior approaches operate in discrete and finite action space \cite{he2012imitation, ratliff2009learning, abbeel2004apprenticeship, argall2009survey}. Some focus on continuous state space \cite{abbeel2005exploration}, but requires a linear model for the system dynamics. 
In contrast, we focus on learning complex smooth functions within continuous action and state spaces. 

\vspace{-0.3em}
One natural approach to tackle the more general setting is to reduce imitation learning to a standard supervised learning problem \cite{syed2010reduction, langford2005relating, lagoudakis2003reinforcement}. However, standard supervised methods assume i.i.d. training and test examples, thus ignoring the distribution mismatch between training and rolled-out trajectories directly applied to sequential learning problems \cite{kakade2002approximately}. Thus a naive supervised learning approach normally leads to unsatisfactory results \cite{ross2010efficient}.

\vspace{-0.05in}
\textbf{Iterative Learning Reductions.}
State-of-the-art learning reductions for imitation learning typically take an iterative approach, where each training round uses standard supervised learning to learn a policy \cite{daume2009search,dagger}.  
In each round $n$, the following happens:
 \vspace{-0.35in}
 \begin{itemize}
 		\itemsep-0.1em
 	\item Given initial state $s_0$ drawn from the starting distribution of states, the learner executes current policy $\pi_n$, resulting in a sequence of states $s_1^n,\ldots,s_T^n$.
   \vspace{-0.04in}
 	\item For each $s_t^n$, a label $\widehat{a}_t^n$ (e.g., expert feedback) is collected indicating what the expert would do given $s_t^n$, resulting in a new dataset $\mathcal{D}_n = \{(s_t,\widehat{a}_t^n) \}$. 
   \vspace{-0.04in}
 	\item The learner integrates $\mathcal{D}_n$ to learn a policy $\pih_n$. The learner updates the current policy to $\pi_{n+1}$ based on $\pih_n$ and $\pi_n$. 
 \end{itemize}
   \vspace{-1em}
The main challenge is controlling for the cascading errors caused by the changing dynamics of the system, i.e., the distribution of states in each $\mathcal{D}_n\sim d_{\pi_n}$. 
A policy trained using $d_{\pi_n}$ induces a different distribution of states than $d_{\pi_n}$, and so is no longer being evaluated on the same distribution as during training.
A principled reduction should (approximately) preserve the i.i.d. relationship between training and test examples. Furthermore the state distribution $d_\pi$ should converge to a stationary distribution. 

\vspace{-0.2em}
The arguably most notable learning reduction approaches for imitation learning are SEARN \cite{daume2009search} and DAgger \cite{dagger}. 
At each round,  SEARN learns a new policy $\pih_n$ and returns a distribution (or mixture) over previously learned policies: $\pi_{n+1} = \beta\pih_n + (1-\beta)\pi_n$ for $\beta\in(0,1)$. 
For appropriately small choices of $\beta$, this stochastic mixing limits the ``distribution drift'' between $\pi_n$ and $\pi_{n+1}$ and can provably guarantee that the performance of $\pi_{n+1}$ does not degrage significantly relative to the expert demonstrations.\footnote{A similar approach was adopted in Conservative Policy Iteration for the MDP setting \cite{kakade2002approximately}.}

\vspace{-0.2em}
DAgger, on the other hand, achieves stability by aggregating a new dataset at each round to learn a new policy from the combined data set $\mathcal{D} \leftarrow\mathcal{D}\cup\mathcal{D}_n$. This aggregation, however, significantly increases the computational complexity and thus is not practical for large problems that require many iterations of learning (since the training time grows super-linearly w.r.t. the number of iterations).

\vspace{-0.2em}
Both SEARN and DAgger showed that only a polynomial number of training rounds is required for convergence to a good policy, but with a dependence on the length of horizon $T$. In particular, to non-trivially bound the total variation distance $\norm{d_{\pi_{new}}-d_{\pi_{old}}}_{1}$ of the state distributions between old and new policies, a learning rate $\beta<\frac{1}{T}$ is required to hold (Lemma 1 of \citet*{daume2009search} and Theorem 4.1 of \citet*{dagger}). As such, systems with very large time horizons might suffer from very slow convergence. 

\vspace{-0.5em}
\textbf{Our Contributions.} Within the context of previous work, our SIMILE approach can be viewed as extending SEARN to smooth policy classes with the following improvements:
\vspace{-2em}
\begin{itemize}
		\itemsep-0.3em
	\item We provide a policy improvement bound that does not depend on the time horizon $T$, and can thus converge much faster. In addition, SIMILE has adaptive learning rate, which can further improve convergence.
	\item For the smooth policy class described in Section \ref{sec:problem}, we show how to generate simulated or ``virtual'' expert feedback in order to guarantee stable learning. This alleviates the need to have continuous access to a dynamic oracle / expert that shows the learner what to do when it is off-track. In this regard, the way SIMILE integrates expert feedback subsumes the set-up from SEARN and DAgger. 
	\item Unlike SEARN, SIMILE returns  fully deterministic policies. Under the continuous setting, deterministic policies are strictly better than stochastic policies as (i) smoothness is critical and (ii) policy sampling requires holding more data during training, which may not be practical for infinite state and action spaces. 
   \item Our theoretical analysis reveals a new sequential prediction setting that yields provably fast convergence, in particular for smooth policy classes on finite-horizon problems.  Existing settings that enjoy such results are limited to Markovian dynamics with discounted future rewards or linear model classes.  
\end{itemize}

\vspace{-0.2in}
\section{Smooth Imitation Learning Algorithm}
\label{sec:method}
\vspace{-0.5em}
Our learning algorithm, called SIMILE (\textbf{S}mooth \textbf{IMI}tation \textbf{LE}arning), is described in Algorithm \ref{algo:simile}. 
At a high level, the process can be described as:
\vspace{-0.5em}
\begin{enumerate}[nolistsep]
\item Start with some initial policy $\pih_0$ (Line 2).
\item At iteration $n$, use $\pi_{n-1}$ to build a new state distribution $\s_n$ and dataset $\mathcal{D}_n = \{(s_t^n,\ah_t^n)\}$ (Lines 4-6).
\item Train $\pih_n = \argmin_{\pi\in\Pi}\mathbb{E}_{s\sim \s_n}\left[\ell_n(\pi(s))\right]$, where $\ell_n$ is the imitation loss (Lines 7-8).  Note that $\ell_n$ needs not be the original $\ell$, but simply needs to converge to it.
\item Interpolate $\pih_n$ and $\pi_{n-1}$ to generate a new deterministic policy $\pi_{n}$ (Lines 9-10). Repeat from Step 2 with $n\leftarrow n+1$ until some termination condition is met.
\end{enumerate}
\begin{algorithm}[tb]
\caption{ SIMILE (\textbf{S}mooth \textbf{IMI}tation \textbf{LE}arning)}
\label{algo:simile}
\begin{algorithmic}[1]
\REQUIRE features $\mathbf{X} = \{x_t\}$, human trajectory $\A^* = \{a_t^*\}$, base routine $\train$, smooth regularizers $h\in\mathcal{H}$ \\
\STATE Initialize $\A_0 \leftarrow \A^*, \s_0\leftarrow \{\left[x_t,a_{t-1}^*\right]\} $, \\ $\qquad\qquad h_0 = \argmin\limits_{h\in\mathcal{H}}\sum\limits_{t=1}^T\norm{a_t^*-h(a_{t-1}^*)}$ \\
\STATE Initial policy $\pi_0 = \pih_0\leftarrow \train(\mathbf{S}_0,\mathbf{A}_0|\enskip h_0)$ \\
\FOR{$n = 1,\ldots, N$}
\STATE $\mathbf{A}_n  = \{a_t^n\} \leftarrow \pi_{n-1}(\mathbf{S}_{n-1})$\label{algo:roll_out} \hfill{//\textit{sequential roll-out} } \\
\STATE $\s_n \leftarrow \{s_t^n = \left[x_t,a_{t-1}^n\right]\} $ \hfill{//$s_t^n = \left[x_{t:t-p},a_{t-1:t-q}\right]$ \label{algo:form_state}} 
\STATE $\widehat{\A}_n = \{\widehat{a}_t^n \} \enskip \forall s_t^n\in\mathbf{S}_n$ \label{algo:collect_feedback} \hfill{// \textit{collect smooth feedback}}  
\STATE $h_n = \argmin\limits_{h\in\mathcal{H}}\sum\limits_{t=1}^T\norm{\widehat{a}_t^n-h(\widehat{a}_{t-1}^n)}$ \label{algo:update_regularizer} \hfill{//\textit{new regularizer}}  \\
\STATE $\pih_n \leftarrow \train(\mathbf{S}_n, \widehat{\mathbf{A}}_n |\enskip h_n)$ \label{algo:learn} \hfill{// \textit{train policy}} \\
\STATE $\beta \leftarrow \beta(\ell(\pih_{n}), \ell(\pi_{n-1}))$ \hfill{//\textit{adaptively set $\beta$ }} \\
\STATE $\pi_n = \beta\pih_n + (1-\beta)\pi_{n-1}$ \label{algo:interpolate} \hfill{// \textit{update policy }} \\
\ENDFOR \\
\OUTPUT Last policy $\pi_N$
\end{algorithmic}
\end{algorithm}
\setlength{\textfloatsep}{6pt}
\vspace{-0.25in}
\textbf{Supervised Learning Reduction.} The actual reduction is in Lines 7-8, where we follow a two-step procedure of first updating the smooth regularize $h_n$, and then training $\pih_n$ via supervised learning.
In other words, $\train$ finds the best $f\in\mathcal{F}$ possible for a fixed $h_n$.  We discuss how to set the training targets $\ah_t^n$ below.
\vspace{-0.3em}
\textbf{Policy Update.} The new policy $\pi_n$ is a deterministic interpolation between the previous $\pi_{n-1}$ and the newly learned $\pih_n$ (Line 10). 
In contrast, for SEARN, $\pi_n$ is a stochastic interploation \cite{daume2009search}. 
Lemma \ref{lem:sto_det} and Corollary \ref{cor:det_better_sto} show that deterministic interpolation converges at least as fast as stochastic for smooth policy classes.

This interpolation step plays two key roles. First, it is a form of myopic or greedy online learning. Intuitively, rolling out $\pi_n$ leads to incidental exploration on the mistakes of $\pi_n$, and so each round of training is focused on refining $\pi_n$. Second, the interpolation in Line 10 ensures a slow drift in the distribution of states from round to round, which preserves an approximate i.i.d. property for the supervised regression subroutine and guarantees convergence.  

However this model interpolation creates an inherent tension between maintaining approximate i.i.d. for valid supervised learning and more aggressive exploration (and thus faster convergence).  For example, SEARN's guarantees only apply for small $\beta < 1/T$. SIMILE circumvents much of this tension via a policy improvement bound that allows $\beta$ to adaptively increase depending on the quality of $\pih_n$ (see Theorem \ref{policy_improvement}), which thus guarantees a valid learning reduction while substantially speeding up convergence.

\vspace{-0.3em}
\textbf{Feedback Generation.} 
We can generate training targets $\ah_t^n$ using  ``virtual'' feedback from simulating expert demonstrations, which has two benefits. First, we need not query the expert $\pis$ at every iteration (as done in DAgger \cite{dagger}).
Continuously acquiring expert demonstrations at every round can be seen as a special case and a more expensive strategy. Second, virtual feedback ensures stable learning, i.e., every $\pih_n$ is a feasible smooth policy.  

\begin{wrapfigure}{r}{0.16\textwidth}
	\centering
	\includegraphics[scale=0.8]{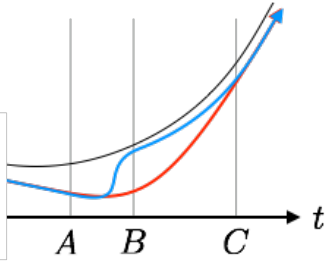}
	\vspace{-0.1in}
\caption{}
	\label{fig:example}
	\vspace{-0.1in}
\end{wrapfigure}
Consider 
Figure \ref{fig:example}, where our policy $\pi_n$ (blue/red) made a mistake at location A, and where we have only a single expert demonstration from $\pis$ (black). Depending on the smoothness requirements of the policy class, we can simulate virtual expert feedback as via either the red line (more smooth) or blue (less smooth) as a tradeoff between squared imitation loss and smoothness.

When the roll-out of $\pi_{n-1}$ (i.e. $\A_n$) differs substantially from $\A^*$, especially during early iterations, using smoother feedback (red instead of blue) can result in more stable learning. We formalize this notion for $\Pi_{\lambda}$ in Proposition \ref{prop:smooth_expert}. Intuitively, whenever $\pi_{n-1}$ makes a mistake, resulting in a ``bad'' state $s_t^n$, the feedback should recommend a smooth correction $\widehat{a}_t^n$ w.r.t. $\A_n$ to make training ``easier'' for the learner.\footnote{A similar idea was proposed \cite{he2012imitation} for DAgger-type algorithm, albeit only for linear model classes.}  The virtual feedback $\widehat{a}_t^n$ should converge to the expert's action $a_t^*$. In practice, we use $\widehat{a}_t^n = \sigma a_{t}^n + (1-\sigma) a_{t}^*$ with $\sigma\rightarrow 0$ as $n$ increases (which satisfies Proposition \ref{prop:smooth_expert}).

\vspace{-0.1in}
\section{Theoretical Results}
\label{sec:theory}
\vspace{-0.5em}
All proofs are deferred to the supplementary material.
\vspace{-0.1in}
\subsection{Stability Conditions}
One natural smoothness condition is that  $\pi(\left[x,a\right])$ should be stable w.r.t. $a$ if $x$ is fixed. 
Consider the camera planning setting: the expert policy $\pis$ should have very small curvature, since constant inputs should correspond to constant actions. 
This motivates Definition $\ref{defn:smooth_policy}$, which requires that $\Pi$ has low curvature given fixed context. We also show that smooth policies per Definition \ref{defn:smooth_policy} lead to stable actions, in the sense that ``nearby'' states are mapped to ``nearby'' actions. The following helper lemma is useful:
\vspace{-0.2em}
\begin{lem}
	\label{lem:H_bound}
For a fixed $x$, define $\pi(\left[x,a\right]) \triangleq\varphi(a)$. If $\varphi$ is non-negative and $H$-smooth w.r.t. $a$., then: 
\begin{equation}
\forall a,a':\ \left(\varphi(a)-\varphi(a^\prime)\right)^2\leq 6H\left(\varphi(a)+\varphi(a^\prime)\right)\norm{a-a^\prime}^2.\nonumber
\end{equation}
\end{lem}

Writing $\pi$ as $\pi(\left[x,a\right])\triangleq\left[\pi^1(\left[x,a\right]),\ldots,\pi^k(\left[x,a\right])\right]$ with each $\pi^i(\left[x,a\right])$ $H$-smooth, Lemma \ref{lem:H_bound} implies $\norm{(\pi(\left[x,a\right]) - \pi(\left[x,a^\prime\right]))} \leq \sqrt{12HR}\norm{a-a^\prime}$ for $R$ upper bounding $\mathcal{A}$. 
Bounded action space means that a sufficiently small $H$ leads to the following stability conditions:
\vspace{-1.3em}
\begin{cond}[Stability Condition 1]
	\label{cond:stability_condition}
$\Pi$ satisfies the Stability Condition 1 if for a fixed input feature $x$, the actions of $\pi$ in states $s=\left[x,a\right]$ and $s^\prime = \left[x,a^\prime\right]$ satisfy $\norm{\pi(s)-\pi(s^\prime)} \leq  \norm{a-a^\prime}$  for all $a,a^\prime\in\mathcal{A}$.
\end{cond}
\vspace{-0.3em}
\begin{cond}[Stability Condition 2]
	\label{cond:self-smooth}
$\Pi$ satisfies Stability Condition 2 if each $\pi$ is $\gamma$-Lipschitz continuous in the action component $a$ with $\gamma<1$. That is, for a fixed $x$ the actions of $\pi$ in states $s=\left[x,a\right]$ and $s^\prime = \left[x,a^\prime\right]$ satisfy $\norm{\pi(s)-\pi(s^\prime)} \leq  \gamma\norm{a-a^\prime}$ for all $a,a^\prime\in\mathcal{A}$.
\end{cond}
\vspace{-0.5em}
These two conditions directly follow from Lemma \ref{lem:H_bound} and assuming sufficiently small $H$. Condition \ref{cond:self-smooth} is mildly stronger than Condition \ref{cond:stability_condition}, and enables proving much stronger policy improvement compared to previous work.

\vspace{-1.0em}
\subsection{Deterministic versus Stochastic}
\vspace{-0.5em}
Given two policies $\pi$ and $\pih$, and interpolation parameter $\beta\in(0,1)$, consider two ways to combine policies:
\vspace{-0.5em}
\begin{enumerate}[nolistsep]
   \item \textbf{stochastic}: $\pi_{sto}(s) = \pih(s)$ with probability $\beta$, and $\pi_{sto}(s) = \pi(s)$ with probability $1-\beta$
   \item \textbf{deterministic}: $\pi_{det}(s) = \beta\pih(s)+(1-\beta)\pi(s)$
\end{enumerate}
Previous learning reduction approaches only use stochastic interpolation \cite{daume2009search,dagger}, whereas SIMILE uses deterministic.  The following result shows that deterministic and stochastic interpolation yield the same expected behavior for smooth policy classes.
\begin{lem}
	\label{lem:sto_det}
Given any starting state $s_0$, sequentially execute $\pi_{det}$ and $\pi_{sto}$ to obtain two separate trajectories $\A = \{a_t\}_{t=1}^T$ and $\tilde{\A} = \{\tilde{a_t}\}_{t=1}^T$ such that $a_t = \pi_{det}(s_t)$ and $\tilde{a}_t = \pi_{sto}(\tilde{s}_t)$, where $s_t = [x_t,a_{t-1}]$ and $\tilde{s}_t = [x_t,\tilde{a}_{t-1}]$. Assuming the policies are stable as per Condition \ref{cond:stability_condition}, we have $\mathbb{E}_{\tilde{\A}}[\tilde{a}_t] = a_t \enskip\forall t = 1,\ldots,T$, where the expectation is taken over all random roll-outs of $\pi_{sto}$.
\end{lem}
\vspace{-0.5em}
Lemma \ref{lem:sto_det} shows that deterministic policy combination (SIMILE) yields unbiased trajectory roll-outs of stochastic policy combination (as done in SEARN \& CPI). This represents a major advantage of SIMILE, since the number of stochastic roll-outs of $\pi_{sto}$ to average to the deterministic trajectory of $\pi_{det}$ is polynomial in the time horizon  $T$, leading to much higher computational complexity. 
Furthermore, for convex imitation loss $\ell_\pi(\pi)$, Lemma \ref{lem:sto_det} and Jensen's inequality yield the following corollary, which states that under convex loss, deterministic policy performs at least no worse than stochastic policy in expectation:
\vspace{-0.3em}
\begin{cor}[Deterministic Policies Perform Better]
	\label{cor:det_better_sto}
For deterministic $\pi_{det}$ and stochastic $\pi_{sto}$ interpolations of two policies $\pi$ and $\pih$, and convex loss $\ell$, we have:
\begin{align*}
\ell_{\pi_{det}}(\pi_{det}) &= \ell_{\pi_{sto}}(\mathbb{E}[\pi_{sto}]) \\
&\leq \mathbb{E}\left[\ell_{\pi_{sto}}(\pi_{sto})\right]
\end{align*}
where the expectation is over all roll-outs of $\pi_{sto}$.
\end{cor}

\begin{rem}
We construct a simple example to show that Condition \ref{cond:stability_condition} may be necessary for iterative learning reductions to converge.
Consider the case where contexts $\X\subset\mathbb{R}$ are either constant or vary neglibly. Expert demonstrations should be constant $\pis([x_n,a^*]) = a^*$ for all $n$. Consider an unstable policy $\pi$ such that $\pi(s) = \pi([x,a]) = ka$ for fixed $k>1$.  The rolled-out trajectory of $\pi$ \textit{diverges} $\pis$ at an exponential rate. Assume optimistically that $\pih$ learns the correct expert behavior, which is  simply $\pih(s) = \pih([x,a]) = a$.  For any $\beta\in(0,1)$, the updated policy $\pi^\prime = \beta\pih+(1-\beta)\pi$ becomes $\pi^\prime([x,a]) = \beta a+(1-\beta)ka$.
Thus the sequential roll-out of the new policy $\pi^\prime$ will also yield an exponential gap from the correct policy.  By induction, the same will be true in all future iterations. 
\end{rem}
\vspace{-1.0em}
\subsection{Policy Improvement}
\vspace{-0.5em}
Our policy improvement guarantee builds upon the analysis from SEARN \cite{daume2009search}, which we extend to using adaptive learning rates $\beta$.
We first restate the main policy improvement result from \citet{daume2009search}.

\begin{lem}[SEARN's policy nondegradation - Lemma 1 from \citet{daume2009search}]
\label{lem:searn_analysis}
Let $\ell_{max}=\sup_{\pi,s}\ell(\pi(s))$, $\pi^\prime$ is defined as $\pi_{sto}$ in lemma \ref{lem:sto_det}. Then for $\beta\in (0,1/T)$:
\begin{equation}
\ell_{\pi^\prime}(\pi^\prime) - \ell_\pi(\pi) \leq \beta T \mathbb{E}_{s\sim d_\pi}\left[\ell(\pih(s))\right]+\frac{1}{2}\beta^2 T^2 \ell_{max}.\nonumber
\end{equation}
\end{lem}
\vspace{-0.1in}
SEARN guarantees that the new policy $\pi^\prime$ does not degrade from the expert $\pis$ by much only if $\beta<1/T$. Analyses of SEARN and other previous iterative reduction methods \cite{dagger, kakade2002approximately, bagnell2003policy, syed2010reduction} rely on bounding the variation distance between $d_\pi$ and $d_{\pi^\prime}$. Three drawbacks of this approach are: (i) non-trivial variation distance bound typically requires $\beta$ to be inversely proportional to time horizon $T$, causing slow convergence; (ii) not easily applicable to the continuous regime; and (iii) except under MDP framework with discounted infinite horizon, previous variation distance bounds do not guarantee monotonic policy improvements (i.e. $\ell_{\pip}(\pip) < \ell_{\pi}(\pi)$).

We provide two levels of guarantees taking advantage of Stability Conditions \ref{cond:stability_condition} and \ref{cond:self-smooth} to circumvent these drawbacks. Assuming the Condition \ref{cond:stability_condition} and convexity of $\ell$, our first result yields a guarantee comparable with SEARN. 
\begin{thm}[T-dependent Improvement]
	\label{first_theorem}
	Assume $\ell$ is convex and $L$-Lipschitz, and Condition \ref{cond:stability_condition} holds. Let $\epsilon = \max\limits_{s\sim d_\pi}\norm{\pih(s) - \pi(s)}$. Then: 
	\begin{equation}
	\ell_{\pi^\prime}(\pi^\prime) - \ell_{\pi}(\pi) \leq \beta\epsilon LT + \beta\left(\ell_{\pi}(\pih) - \ell_{\pi}(\pi)\right).
	\end{equation}
	In particular, choosing $\beta\in(0,1/T)$ yields:
	\begin{equation}
	\ell_{\pi^\prime}(\pi^\prime) - \ell_{\pi}(\pi) \leq \epsilon L + \beta\left(\ell_{\pi}(\pih) - \ell_{\pi}(\pi)\right).
	\label{eqn:horizon_dependent}
			\vspace{-0.05in}
	\end{equation}
\end{thm}
Similar to SEARN, Theorem \ref{first_theorem} also requires $\beta\in (0,1/T)$ to ensure the RHS of \eqref{eqn:horizon_dependent} stays small.  However, note that the reduction term $\beta(\ell_{\pi}(\pih) -\ell_{\pi}(\pi))$ allows the bound to be strictly negative if the policy $\pih$ trained on $d_\pi$ significantly improves on $\ell_{\pi}(\pi)$ (i.e., guaranteed policy improvement). We observe empirically that this often happens, especially in early iterations of training. 

\vspace{-0.1em}
Under the mildly stronger Condition \ref{cond:self-smooth}, we remove the dependency on the time horizon $T$, which represents a much stronger guarantee compared to previous work.  
\vspace{-0.1em}
\begin{thm}[Policy Improvement] 
	\label{policy_improvement}
	Assume $\ell$ is convex and $L$-Lipschitz-continuous, and Condition \ref{cond:self-smooth} holds. Let $\epsilon = \max\limits_{s\sim d_\pi}\norm{\pih(s) - \pi(s)}$.  Then for $\beta\in(0,1)$: 
	\begin{equation}
	\ell_{\pip}(\pip) - \ell_{\pi}(\pi) \leq \frac{\beta\gamma\epsilon L}{(1-\beta)(1-\gamma)} + \beta(\ell_{\pi}(\pih) - \ell_{\pi}(\pi)).
	\label{eqn:policy_improvement_bound}
	\end{equation}
\end{thm}
\begin{cor}[Monotonic Improvement]
	\label{cor:monotonic_improvement}
Following the notation from Theorem \ref{policy_improvement}, let $\Delta =\ell_{\pi}(\pi) - \ell_{\pi}(\pih)$ and $\delta = \frac{\gamma\epsilon L}{1-\gamma}$. Then choosing step size $\beta = \frac{\Delta-\delta}{2\Delta}$, we have: 
\begin{equation}
\ell_{\pi^\prime}(\pi^\prime)-\ell_{\pi}(\pi) \leq -\frac{(\Delta-\delta)^2}{2(\Delta+\delta)}.
\label{cor:monotonic_bound}
\end{equation}
\end{cor}
\vspace{-0.2em}
The terms $\epsilon$ and $\ell_{\pi}(\pih) - \ell_{\pi}(\pi)$ on the RHS of \eqref{eqn:horizon_dependent} and \eqref{eqn:policy_improvement_bound} come from the learning reduction, as they measure the ``distance''  between  $\pih$ and $\pi$ on the state distribution induced by $\pi$ (which forms the dataset to train $\pih$). In practice, both terms can be empirically estimated from the training round, thus allowing an estimate of $\beta$ to minimize the bound. 

Theorem \ref{policy_improvement} justifies using an adaptive and more aggressive interpolation parameter $\beta$ to update policies. In the worst case, setting $\beta$ close to $0$ will ensure the bound from  \eqref{eqn:policy_improvement_bound} to be close to $0$, which is consistent with SEARN's result. More generally, monotonic policy improvement can be guaranteed for appropriate choice of $\beta$, as seen from Corollary \ref{cor:monotonic_improvement}. This strict policy improvement was not possible under previous iterative learning reduction approaches such as SEARN and DAgger, and is enabled in our setting due to exploiting the smoothness conditions.
\vspace{-0.6em}
\subsection{Smooth Feedback Analysis}
\vspace{-0.5em}
\textbf{Smooth Feedback Does Not Hurt:} Recall from Section \ref{sec:method} that one way to simulate ``virtual'' feedback for training a new $\pih$ is to set the target $\hat{a}_t=\sigma a_t+(1-\sigma) a_t^*$ for $\sigma\in(0,1)$, where smooth feedback corresponds to $\sigma\rightarrow 1$.
To see that simulating smooth ``virtual'' feedback target does not hurt the training progress, we alternatively view SIMILE as performing gradient descent in a smooth function space \cite{mason1999functional}. Define the cost functional $C:\Pi\rightarrow\mathbb{R}$ over policy space to be the average imitation loss over $\mathcal{S}$ as $C(\pi) = \int\limits_{\mathcal{S}}\norm{\pi(s)-\pis(s)}^2 dP(s)$. The gradient (G\^ateaux derivative) of $C(\pi)$ w.r.t. $\pi$ is:
$$\nabla C(\pi)(s) = \frac{\partial C(\pi+\alpha\delta_{s})}{\partial\alpha}\Bigr\vert_{\alpha=0} = 2(\pi(s) - \pis(s)),$$ 
where $\delta_s$ is Dirac delta function centered at s. By first order approximation $C(\pip) = C(\beta\pih+(1-\beta)\pi) = C(\pi +\beta(\pih-\pi)) \approx C(\pi) +  \beta\inner{\nabla C(\pi), \pih-\pi}$. Like traditional gradient descent, we want to choose $\pih$ such that the update moves the functional along the direction of negative gradient. In other words, we want to learn $\pih\in\Pi$ such that $\inner{\nabla C(\pi),\pih-\pi} \ll 0$. We can evaluate this inner product along the states induced by $\pi$. We thus have the estimate:
\begin{align*}
			\vspace{-0.1in}
\inner{\nabla C(\pi), \pih - \pi} 
&\approx\frac{2}{T}\sum_{t=1}^{T}(\pi(s_t)-\pis(s_t))(\pih(s_t) -\pi(s_t)) \\
&= \frac{2}{T}\sum_{t=1}^{T}(a_t - a^*_t)(\pih([x_t,a_{t-1}])-a_t).
\end{align*}  
Since we want $\inner{\nabla C(\pi), \pih - \pi}<0$, this motivates the construction of new data set $\mathcal{D}$ with states $\{[x_t,a_{t-1}]\}_{t=1}^T$ and labels $\{\widehat{a}_t\}_{t=1}^T$ to train a new policy $\pih$, where we want $(a_t-a_t^*)(\widehat{a}_t-a_t)<0$. A sufficient solution is to set target $\widehat{a}_t = \sigma a_t+(1-\sigma)a_t^*$ (Section \ref{sec:method}), as this will point the gradient in negative direction, allowing the learner to make progress. 




\textbf{Smooth Feedback is Sometimes Necessary:} When the current policy performs poorly, smooth virtual feedback may be required to ensure stable learning, i.e. producing a feasible smooth policy at each training round. 
We formalize this notion of feasibility by considering the smooth policy class $\Pi_{\lambda}$ in Example \ref{example:smooth_policy_class}.  Recall that smooth regularization of $\Pi_{\lambda}$ via $\mathcal{H}$ encourages the next action to be close to the previous action. Thus a natural way to measure smoothness of $\pi\in\Pi_{\lambda}$ is via the average first order difference of consecutive actions $\frac{1}{T}\sum_{t=1}^T \norm{a_t - a_{t-1}}$. In particular, we want to explicitly constrain this  difference relative to the expert trajectory $\frac{1}{T}\sum_{t=1}^T \norm{a_t - a_{t-1}} \leq \eta$ at each iteration, where $\eta\propto \frac{1}{T}\sum_{t=1}^T\norm{a_t^* - a_{t-1}^*}$. 

When $\pi$ performs poorly, i.e. the "average gap" between current trajectory $\{a_t\}$ and $\{a_t^*\}$ is large, the training target for $\pih$ should be lowered to ensure learning a smooth policy is feasible, as inferred from the following proposition. In practice, we typically employ smooth virtual feedback in early iterations when policies tend to perform worse. 
\vspace{-0.5em}
\begin{prop}[]
	\label{prop:smooth_expert}
	Let $\omega$ be the average supervised training error from $\mathcal{F}$, i.e. $\omega = \min\limits_{f\in\mathcal{F}}\mathbb{E}_{x\sim\mathcal{X}}\left[\norm{f([x,0])-a^*}\right]$. Let the rolled-out trajectory of current policy $\pi$ be $\{a_t\}$. If the average gap between $\pi$ and $\pis$ is such that $\mathbb{E}_{t\sim \text{Uniform}[1:T]}\left[\norm{a_t^*-a_{t-1}}\right] \geq 3\omega+\eta(1+\lambda)$, then using $\{a_t^*\}$ as feedback will cause the trained policy $\pih$ to be non-smooth, i.e.:
	\begin{equation}
	\mathbb{E}_{t\sim \text{Uniform}[1:T]}\left[\norm{\hat{a}_t-\hat{a}_{t-1}}\right] \geq \eta,
	\end{equation}
	for $\{\hat{a}_t\}$ the rolled-out trajectory of $\pih$.
\end{prop}
\vspace{-0.15in}
\section{Experiments}
\vspace{-0.5em}
\textbf{Automated Camera Planning.} 
We evaluate SIMILE in a case study of automated camera planning for sport broadcasting \cite{chen2015mimicking,smooth_camera}.
Given noisy tracking of players as raw input data $\{x_t\}_{t=1}^T$, and demonstrated pan camera angles from professional human operator as $\{a_t^*\}_{t=1}^T$, the goal is to learn a policy $\pi$ that produces trajectory $\{a_t\}_{t=1}^T$ that is both smooth and accurate relative to $\{a_t^*\}_{t=1}^T$. Smoothness is critical in camera control: fluid movements which maintain adequate framing are preferable to jittery motions which constantly pursue perfect tracking \cite{gaddam2015cameraman}. In this setting, time horizon $T$ is the duration of the event multiplied by rate of sampling. Thus $T$ tends to be very large. 

\vspace{-0.5em}
\textbf{Smooth Policy Class.}
We use a smooth policy class following Example \ref{example:autoregressor_policy_class}: regression tree ensembles $\mathcal{F}$ regularized by a class of linear autoregressor functions $\mathcal{H}$ \cite{smooth_camera}. See Appendix \ref{sec:tree} for more details.  

\vspace{-0.5em}
\textbf{Summary of Results.}  
\vspace{-0.5em}
\begin{itemize}[nolistsep]
\item Using our smooth policy class leads to dramatically smoother trajectories than not regularizing using $\mathcal{H}$. 
\item Using our adaptive learning rate leads to much faster convergence compared to conservative learning rates from SEARN \cite{daume2009search}.
\item Using smooth feedback ensures stable learning of smooth policies at each iteration.
\item Deterministic policy interpolation performs better than stochastic interpolation used in SEARN.
\end{itemize}
\begin{figure}[t]
	\centering
	\includegraphics[scale=0.5]{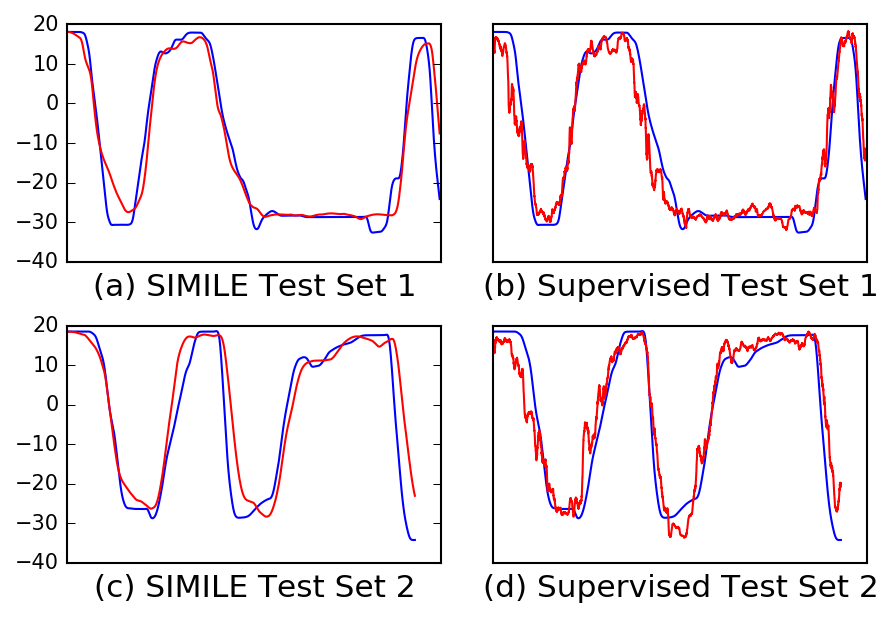}
	\vspace{-0.18in}
	\caption{Expert (blue) and predicted (red) camera pan angles. Left: SIMILE with $<$10 iterations. Right: non-smooth policy.}
	\label{fig:test_set_performance}
\end{figure}
\begin{figure}[t]
\vspace{-0.5em}
	\centering
	\includegraphics[scale=0.25]{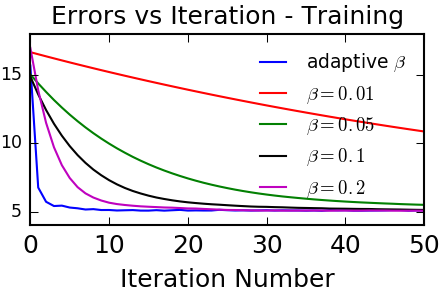}
	\includegraphics[scale=0.25]{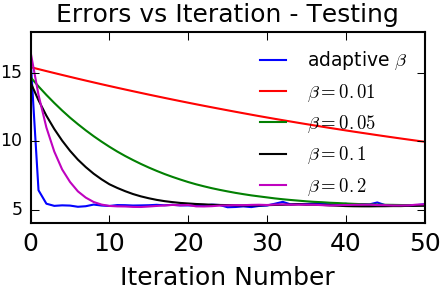}
	\vspace{-0.15in}
	\caption{Adaptive versus fixed interpolation parameter $\beta$.}
	\label{fig:adaptive_vs_fixed_beta}
\end{figure}
\textbf{Smooth versus Non-Smooth Policy Classes.} 
Figure \ref{fig:test_set_performance} shows a comparison of using a smooth policy class versus a non-smooth one (e.g., not using $\mathcal{H}$).
We see that our approach can reliably learn to predict trajectories that are both smooth and accurate.

\vspace{-0.5em}
\textbf{Adaptive vs. Fixed $\beta$:}
One can, in principle, train using SEARN, which requires a very conservative $\beta$ in order to guarantee convergence. In contrast, SIMILE adaptively selects $\beta$ based on relative empirical loss of $\pi$ and $\pih$ (Line 9 of Algorithm \ref{algo:simile}). Let $\texttt{error}(\pih)$ and $\texttt{error}(\pi)$ denote the mean-squared errors of rolled-out trajectories $\{\hat{a}_t\}$, $\{a_t\}$, respectively, w.r.t. ground truth $\{a_t^*\}$. We can set $\beta$ as:
\begin{equation}
\hat{\beta} = \frac{\texttt{error}(\pi)}{\texttt{error}(\pih) + \texttt{error}(\pi)},
\label{eqn:beta}
\end{equation}
which encourages the learner to disregard bad policies when interpolating, thus allowing fast convergence to a good policy (see Theorem \ref{policy_improvement}). 
Figure \ref{fig:adaptive_vs_fixed_beta} compares the convergence rate of SIMILE using adaptive $\beta$ versus conservative fixed values of $\beta$ commonly used in SEARN \cite{daume2009search}. We see that adaptively choosing $\beta$ enjoys substantially faster convergence. Note that very large fixed $\beta$ may overshoot and worsen the combined policy after a few initial improvements.
\begin{figure}[t]
\centering
	\includegraphics[scale=0.5]{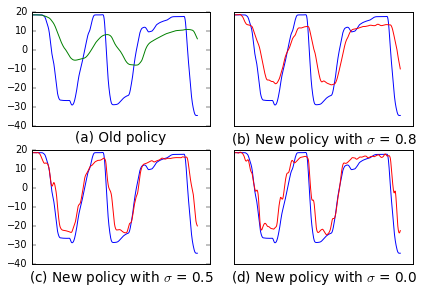}
		\vspace{-0.2in}
	\caption{Comparing different values of $\sigma$.}
	\label{fig:feedback_varying_sigma}
\end{figure}

\vspace{-0.5em}
\textbf{Smooth Feedback Generation:}
We set the target labels to $\hat{a}_t^n = \sigma a_t^n+(1-\sigma)a_t^*$ for $0<\sigma<1$ (Line 6 of Algorithm \ref{algo:simile}). 
Larger $\sigma$ corresponds to smoother ($\hat{a}_t^n$ is closer to $a_{t-1}^n$) but less accurate target (further from $a_t^*$), as seen in Figure \ref{fig:feedback_varying_sigma}. 
Figure \ref{fig:smooth_vs_distance_error} shows the trade-off between 
\begin{wrapfigure}{r}{0.25\textwidth}
		\vspace{-0.2in}
	\centering
	\includegraphics[scale=0.3]{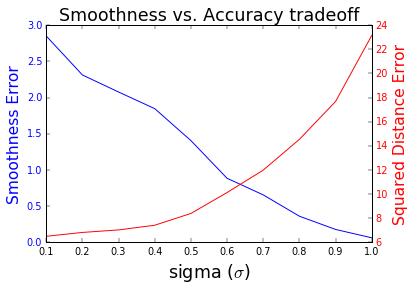}
	\vspace{-0.35in}
	\caption{}
	\label{fig:smooth_vs_distance_error}
	\vspace{-0.15in}
\end{wrapfigure}
smoothness loss (blue line, measured by first order difference in Proposition \ref{prop:smooth_expert}) and imitation loss (red line, measured by mean squared distance) for varying $\sigma$. We navigate this trade-off by setting $\sigma$ closer to 1 in early iterations, and have $\sigma\rightarrow 0$ as $n$ increases. This ``gradual increase'' produces more stable policies, especially during early iterations where the learning policy tends to perform poorly (as formalized in Proposition \ref{prop:smooth_expert}). In Figure \ref{fig:feedback_varying_sigma}, when the initial policy (green trajectory) has poor performance, setting smooth targets (Figure \ref{fig:feedback_varying_sigma}b) allows learning a smooth policy in the subsequent round, in contrast to more accurate but less stable performance of ``difficult'' targets with low $\sigma$ (Figure \ref{fig:feedback_varying_sigma}c-d).
Figure \ref{fig:learning_progress} visualizes the behavior of the the intermediate policies learned by SIMILE, where we can see that each intermediate policy is a smooth policy.

\begin{figure}[t]
\centering
	\hspace{-.05in}\includegraphics[scale=0.5]{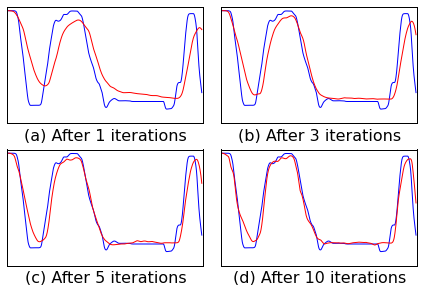}
	\vspace{-0.2in}
	\caption{Performance after different number of iterations.}
	\label{fig:learning_progress}
\end{figure}

\vspace{-0.5em}
\textbf{Deterministic vs. Stochastic Interpolation:}
Finally, we evaluate the benefits of using deterministic policy averaging intead of stochastically combine different policies, as done in SEARN. To control for other factors, we set $\beta$ to a fixed value of $0.5$, and keep the new training dataset $\mathcal{D}_n$ the same for each iteration $n$. The average imitation loss of stochastic policy sampling are evaluated after 50 stochastic roll-outs at each iterations. This average stochastic policy error tends to be higher compared to the empirical error of the deterministic trajectory, as seen from Figure \ref{fig:det_vs_sto_error}, and confirms our finding from Corollary \ref{cor:det_better_sto}. 
\vspace{-0.15in}
\begin{figure}[h]
\centering
 \includegraphics[scale=0.55]{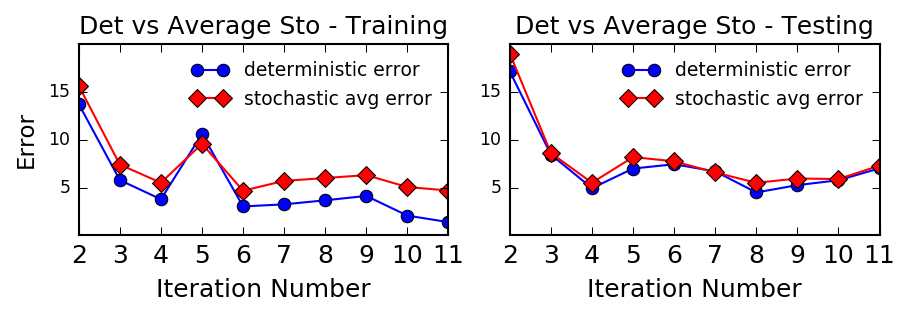}
		\vspace{-0.3in}
	\caption{Deterministic policy error vs. average stochastic policy error for $\beta =0.5$ and 50 roll-outs of the stochastic policies.}
	\label{fig:det_vs_sto_error}
\end{figure}

\vspace{-2em}
\section{Conclusion}
\vspace{-0.5em}
We formalized the problem of smooth imitation learning for online sequence prediction, which is a variant of imitation learning that uses a notion of a smooth policy class.
We proposed SIMILE (\textbf{S}mooth \textbf{IMI}tation \textbf{LE}arning), which is an iterative learning reduction approach to learning smooth policies from expert demonstrations in a continuous and dynamic environment.
   SIMILE utilizes an adaptive learning rate that provably allows much faster convergence compared to previous learning reduction approaches, and also enjoys better sample complexity than previous work by being fully deterministic and allowing for virtual simulation of training labels. We validated the efficiency and practicality of our approach on a setting of online camera planning.

\newpage
\bibliography{icml_paper}

\begin{thebibliography}{22}
\providecommand{\natexlab}[1]{#1}
\providecommand{\url}[1]{\texttt{#1}}
\expandafter\ifx\csname urlstyle\endcsname\relax
  \providecommand{\doi}[1]{doi: #1}\else
  \providecommand{\doi}{doi: \begingroup \urlstyle{rm}\Url}\fi

\bibitem[Abbeel \& Ng(2004)Abbeel and Ng]{abbeel2004apprenticeship}
Abbeel, Pieter and Ng, Andrew~Y.
\newblock Apprenticeship learning via inverse reinforcement learning.
\newblock In \emph{International Conference on Machine Learning {(ICML)}},
  2004.

\bibitem[Abbeel \& Ng(2005)Abbeel and Ng]{abbeel2005exploration}
Abbeel, Pieter and Ng, Andrew~Y.
\newblock Exploration and apprenticeship learning in reinforcement learning.
\newblock In \emph{International Conference on Machine Learning {(ICML)}},
  2005.

\bibitem[Argall et~al.(2009)Argall, Chernova, Veloso, and
  Browning]{argall2009survey}
Argall, Brenna~D, Chernova, Sonia, Veloso, Manuela, and Browning, Brett.
\newblock A survey of robot learning from demonstration.
\newblock \emph{Robotics and autonomous systems}, 57\penalty0 (5):\penalty0
  469--483, 2009.

\bibitem[Bagnell et~al.(2003)Bagnell, Kakade, Schneider, and
  Ng]{bagnell2003policy}
Bagnell, J~Andrew, Kakade, Sham~M, Schneider, Jeff~G, and Ng, Andrew~Y.
\newblock Policy search by dynamic programming.
\newblock In \emph{Neural Information Processing Systems {(NIPS)}}, 2003.

\bibitem[Caruana \& Niculescu-Mizil(2006)Caruana and
  Niculescu-Mizil]{caruana2006empirical}
Caruana, Rich and Niculescu-Mizil, Alexandru.
\newblock An empirical comparison of supervised learning algorithms.
\newblock In \emph{International Conference on Machine Learning {(ICML)}},
  2006.

\bibitem[Chen \& Carr(2015)Chen and Carr]{chen2015mimicking}
Chen, Jianhui and Carr, Peter.
\newblock Mimicking human camera operators.
\newblock In \emph{IEEE Winter Conference Applications of Computer Vision
  (WACV)}, 2015.

\bibitem[Chen et~al.(2016)Chen, Le, Carr, Yue, and Little]{smooth_camera}
Chen, Jianhui, Le, Hoang~M., Carr, Peter, Yue, Yisong, and Little, James~J.
\newblock Learning online smooth predictors for real-time camera planning using
  recurrent decision trees.
\newblock In \emph{IEEE Conference on Computer Vision and Pattern Recognition
  {(CVPR)}}, 2016.

\bibitem[Criminisi et~al.(2012)Criminisi, Shotton, and
  Konukoglu]{DecisionForests}
Criminisi, Antonio, Shotton, Jamie, and Konukoglu, Ender.
\newblock Decision forests: A unified framework for classification, regression,
  density estimation, manifold learning and semi-supervised learning.
\newblock \emph{Foundations and Trends in Computer Graphics and Vision},
  7\penalty0 (2--3):\penalty0 81--227, 2012.

\bibitem[Daum{\'e}~III et~al.(2009)Daum{\'e}~III, Langford, and
  Marcu]{daume2009search}
Daum{\'e}~III, Hal, Langford, John, and Marcu, Daniel.
\newblock Search-based structured prediction.
\newblock \emph{Machine learning}, 75\penalty0 (3):\penalty0 297--325, 2009.

\bibitem[Gaddam et~al.(2015)Gaddam, Eg, Langseth, Griwodz, and
  Halvorsen]{gaddam2015cameraman}
Gaddam, Vamsidhar~Reddy, Eg, Ragnhild, Langseth, Ragnar, Griwodz, Carsten, and
  Halvorsen, P{\aa}l.
\newblock The cameraman operating my virtual camera is artificial: Can the
  machine be as good as a human\&quest.
\newblock \emph{ACM Transactions on Multimedia Computing, Communications, and
  Applications (TOMM)}, 11\penalty0 (4):\penalty0 56, 2015.

\bibitem[He et~al.(2012)He, Eisner, and Daume]{he2012imitation}
He, He, Eisner, Jason, and Daume, Hal.
\newblock Imitation learning by coaching.
\newblock In \emph{Neural Information Processing Systems {(NIPS)}}, 2012.

\bibitem[Jain et~al.(2013)Jain, Wojcik, Joachims, and Saxena]{jain2013learning}
Jain, Ashesh, Wojcik, Brian, Joachims, Thorsten, and Saxena, Ashutosh.
\newblock Learning trajectory preferences for manipulators via iterative
  improvement.
\newblock In \emph{Neural Information Processing Systems {(NIPS)}}, 2013.

\bibitem[Kakade \& Langford(2002)Kakade and Langford]{kakade2002approximately}
Kakade, Sham and Langford, John.
\newblock Approximately optimal approximate reinforcement learning.
\newblock In \emph{International Conference on Machine Learning {(ICML)}},
  2002.

\bibitem[Lagoudakis \& Parr(2003)Lagoudakis and
  Parr]{lagoudakis2003reinforcement}
Lagoudakis, Michail and Parr, Ronald.
\newblock Reinforcement learning as classification: Leveraging modern
  classifiers.
\newblock In \emph{International Conference on Machine Learning {(ICML)}},
  2003.

\bibitem[Langford \& Zadrozny(2005)Langford and Zadrozny]{langford2005relating}
Langford, John and Zadrozny, Bianca.
\newblock Relating reinforcement learning performance to classification
  performance.
\newblock In \emph{International Conference on Machine Learning {(ICML)}},
  2005.

\bibitem[Mason et~al.(1999)Mason, Baxter, Bartlett, and
  Frean]{mason1999functional}
Mason, Llew, Baxter, Jonathan, Bartlett, Peter~L, and Frean, Marcus.
\newblock Functional gradient techniques for combining hypotheses.
\newblock In \emph{Neural Information Processing Systems {(NIPS)}}, 1999.

\bibitem[Ratliff et~al.(2009)Ratliff, Silver, and Bagnell]{ratliff2009learning}
Ratliff, Nathan, Silver, David, and Bagnell, J.~Andrew.
\newblock Learning to search: Functional gradient techniques for imitation
  learning.
\newblock \emph{Autonomous Robots}, 27\penalty0 (1):\penalty0 25--53, 2009.

\bibitem[Ross \& Bagnell(2010)Ross and Bagnell]{ross2010efficient}
Ross, St{\'e}phane and Bagnell, Drew.
\newblock Efficient reductions for imitation learning.
\newblock In \emph{Conference on Artificial Intelligence and Statistics
  {(AISTATS)}}, 2010.

\bibitem[Ross et~al.(2011)Ross, Gordon, and Bagnell]{dagger}
Ross, Stephane, Gordon, Geoff, and Bagnell, J.~Andrew.
\newblock A reduction of imitation learning and structured prediction to
  no-regret online learning.
\newblock In \emph{Conference on Artificial Intelligence and Statistics
  {(AISTATS)}}, 2011.

\bibitem[Srebro et~al.(2010)Srebro, Sridharan, and
  Tewari]{srebro2010smoothness}
Srebro, Nathan, Sridharan, Karthik, and Tewari, Ambuj.
\newblock Smoothness, low noise and fast rates.
\newblock In \emph{Neural Information Processing Systems {(NIPS)}}, 2010.

\bibitem[Syed \& Schapire(2010)Syed and Schapire]{syed2010reduction}
Syed, Umar and Schapire, Robert~E.
\newblock A reduction from apprenticeship learning to classification.
\newblock In \emph{Neural Information Processing Systems {(NIPS)}}, 2010.

\bibitem[Wold(1939)]{wold1939study}
Wold, Herman.
\newblock A study in the analysis of stationary time series, 1939.

\end{thebibliography}
\bibliographystyle{icml2016}

\newpage

\appendix
\twocolumn[
\icmltitle{Supplementary Materials for \\Smooth Imitation Learning for Online Sequence Prediction}
]
\setlength{\abovedisplayskip}{4pt}
\setlength{\belowdisplayskip}{4pt}
\setlength{\abovedisplayshortskip}{3pt}
\setlength{\belowdisplayshortskip}{3pt}
\section{Detailed Theoretical Analysis and Proofs}
\subsection{Proof of lemma \ref{lem:H_bound}}
\begin{lemma_statement}{(Lemma \ref{lem:H_bound})}
	For a fixed $x$, define $\pi(\left[x,a\right]) \triangleq\varphi(a)$. If $\varphi$ is non-negative and $H$-smooth w.r.t. $a$., then: 
	\begin{equation}
	\forall a,a':\ \left(\varphi(a)-\varphi(a^\prime)\right)^2\leq 6H\left(\varphi(a)+\varphi(a^\prime)\right)\norm{a-a^\prime}^2.\nonumber
	\end{equation}
\end{lemma_statement}
The proof of Lemma \ref{lem:H_bound} rests on 2 properties of $H$-smooth functions (differentiable) in $\mathbb{R}^1$, as stated below
\begin{lem}[Self-bounding property of Lipschitz-smooth functions]
	\label{lem:self-bounding}
	Let $\phi:\mathbb{R}\rightarrow\mathbb{R}$ be an $H$-smooth non-negative function. Then for all $a\in\mathbb{R}$: $\enskip \abs{\nabla\phi(a)} \leq \sqrt{4H\phi(a)}$ 
\end{lem}
\begin{proof}
By mean value theorem, for any $a, a^\prime$ we have $\exists \enskip\eta\in\left(a,a^\prime\right)$ (or $\left(a^\prime,a\right)$) such that $\phi(a^\prime) = \phi(a)+\nabla\phi(\eta)(a^\prime - a)$. Since $\phi$ is non-negative,
\begin{align*}
0\leq \phi(a^\prime) &= \phi(a)+\nabla\phi(a)(a^\prime-a) \\
&\qquad+(\nabla\phi(\eta)-\nabla\phi(a))(a^\prime-a) \\
&\leq \phi(a)+\nabla\phi(a)(a^\prime-a) + H\abs{\eta-a}\abs{a^\prime-a} \\
&\leq \phi(a)+\nabla\phi(a)(a^\prime-a) + H\abs{a^\prime-a}^2
\end{align*}
Choosing $a^\prime = a - \frac{\nabla\phi(a)}{2H}$ proves the lemma.
\end{proof}
\begin{lem}[1-d Case \cite{srebro2010smoothness}]
	\label{lem:1d_H_bound}
	Let $\phi:\mathbb{R}\rightarrow\mathbb{R}$ be an $H$-smooth non-negative function. Then for all $a, a^\prime\in\mathbb{R}$:
	\begin{equation*}
\left(\phi(a)-\phi(a^\prime)\right)^2 \leq 6H\left(\phi(a)+\phi(a^\prime)\right)\left(a-a^\prime\right)^2
	\end{equation*}
\end{lem}
\begin{proof}
As before, $\exists \eta\in(a,a^\prime)$ such that $\phi(a^\prime) - \phi(a) = \nabla\phi(\eta)(a^\prime - a)$. By assumption of $\phi$, we have $\abs{\nabla\phi(\eta) - \nabla\phi(a)}\leq H\abs{\eta-a}\leq H\abs{a^\prime-a}$. Thus we have:
\begin{equation}
\label{eqn:phi_gradient_bound}
\abs{\nabla\phi(\eta)} \leq \abs{\nabla\phi(a)} + H\abs{a-a^\prime}
\end{equation}
Consider two cases:

Case 1: If $\abs{a - a^\prime} \leq \frac{\abs{\nabla\phi(a)}}{5H}$, then by equation \ref{eqn:phi_gradient_bound} we have $\abs{\nabla\phi(\eta)}\leq 6/5\abs{\nabla\phi(a)}$. Thus
\begin{align*}
\left(\phi(a)-\phi(a^\prime)\right)^2  &= \left(\nabla\phi(\eta)\right)^2 \left(a-a^\prime\right)^2 \\
&\leq \frac{36}{25}\left(\nabla\phi(a)\right)^2 \left(a-a^\prime\right)^2 \\
&\leq \frac{144}{25}H\phi(a) \left(a-a^\prime\right)^2
\end{align*}
by lemma \ref{lem:self-bounding}. Therefore, $\left(\phi(a)-\phi(a^\prime)\right)^2 \leq 6H\phi(a)\left(a-a^\prime\right)^2 \leq 6H\left(\phi(a)+\phi(a^\prime)\right)\left(a-a^\prime\right)^2$

Case 2: If $\abs{a - a^\prime} > \frac{\abs{\nabla\phi(a)}}{5H}$, then equation \ref{eqn:phi_gradient_bound} gives $\abs{\nabla\phi(\eta)} \leq 6H\abs{a-a^\prime}$. Once again
\begin{align*}
\left(\phi(a)-\phi(a^\prime)\right)^2 &= \left(\phi(a)-\phi(a^\prime)\right)\nabla\phi(\eta)  \left(a-a^\prime\right) \\
&\leq \abs{\left(\phi(a)-\phi(a^\prime)\right)}\abs{\nabla\phi(\eta)}\abs{\left(a-a^\prime\right)} \\
&\leq \abs{\left(\phi(a)-\phi(a^\prime)\right)} \left(6H\left(a-a^\prime\right)^2\right) \\
&\leq 6H\left(\phi(a)+\phi(a^\prime)\right)\left(a-a^\prime\right)^2
\end{align*}
\end{proof}
\begin{proof}[Proof of Lemma \ref{lem:H_bound}]
The extension to the multi-dimensional case is straightforward. For any $a, a^\prime\in\mathbb{R}^k$, consider the function $\phi:\mathbb{R}\rightarrow\mathbb{R}$ such that $\phi(t) = \varphi((1-t)a+ta^\prime)$, then $\phi$ is a differentiable, non-negative function and $\nabla_{t}(\phi(t)) = \inner{\nabla\varphi(a+t(a^\prime-a)), a^\prime-a}$. Thus:
\begin{align*}
&\abs{\phi^\prime(t_1) - \phi^\prime(t_2)} = \vert\langle\nabla\varphi(a+t_1(a^\prime-a)) - \\ &\qquad\qquad\qquad\qquad\qquad\nabla\varphi(a+t_2(a^\prime-a)), a^\prime-a \rangle\vert \\
&\leq \norm{\nabla\varphi(a+t_1(a^\prime-a)) - \nabla\varphi(a+t_2(a^\prime-a))}_{*}\norm{a^\prime - a} \\
&\leq H\abs{t_1-t_2}\norm{a-a^\prime}^2
\end{align*}
Therefore $\phi$ is an $H\norm{a-a^\prime}^2$-smooth function in $\mathbb{R}$. Apply lemma \ref{lem:1d_H_bound} to $\phi$, we have:
\begin{equation*}
\left(\phi(1)-\phi(0)\right)^2 \leq 6H\norm{a-a^\prime}^2\left(\phi(1)+\phi(0)\right)(1-0)^2
\end{equation*}
which is the same as $(\varphi(a)-\varphi(a^\prime))^2\leq 6H(\varphi(a)+\varphi(a^\prime))\norm{a-a^\prime}^2$
\end{proof}
\subsection{Proof of lemma \ref{lem:sto_det}}
\begin{lemma_statement}{(Lemma \ref{lem:sto_det})}
	Given any starting state $s_0$, sequentially execute $\pi_{det}$ and $\pi_{sto}$ to obtain two separate trajectories $\A = \{a_t\}_{t=1}^T$ and $\tilde{\A} = \{\tilde{a_t}\}_{t=1}^T$ such that $a_t = \pi_{det}(s_t)$ and $\tilde{a}_t = \pi_{sto}(\tilde{s}_t)$, where $s_t = [x_t,a_{t-1}]$ and $\tilde{s}_t = [x_t,\tilde{a}_{t-1}]$. Assuming the policies are stable as per Condition \ref{cond:stability_condition}, we have $\mathbb{E}_{\tilde{\A}}[\tilde{a}_t] = a_t \enskip\forall t = 1,\ldots,T$, where the expectation is taken over all random roll-outs of $\pi_{sto}$.
\end{lemma_statement}
\begin{proof}
	Given a starting state $s_0$, we prove by induction that $\mathbb{E}_{\tilde{\A}}[\tilde{a}_t] = a_t$. 
	
	It is easily seen that the claim is true for $t=1$. 
	
	Now assuming that $\mathbb{E}_{\tilde{\A}}[\tilde{a}_{t-1}] = a_{t-1}$. We have 
	\begin{align*}
\mathbb{E}_{\tilde{\A}}[\tilde{a}_t] &= \mathbb{E}_{\tilde{\A}}[\mathbb{E}[\tilde{a}_t|\tilde{s}_t]] \\
&= \mathbb{E}_{\tilde{\A}}[\beta\pih(\tilde{s}_t)+(1-\beta)\pi(\tilde{s}_t)] \\ 
&= \beta\mathbb{E}_{\tilde{\A}}[\pih(\tilde{s}_t)] + (1-\beta)\mathbb{E}_{\tilde{\A}}[\pi(\tilde{s}_t)]
	\end{align*}
	Thus:
	\begin{align*}
	\norm{\mathbb{E}_{\tilde{\A}}[\tilde{a}_t] - a_t} &= \norm{\mathbb{E}_{\tilde{\A}}[\tilde{a}_t]-\beta\pih(s_t)-(1-\beta)\pi(s_t)} \nonumber \\
	&=\Vert\beta\mathbb{E}_{\tilde{\A}}[\pih(\tilde{s}_t)] + (1-\beta)\mathbb{E}_{\tilde{\A}}[\pi(\tilde{s}_t)] \\ &\qquad\qquad -\beta\pih(s_t)-(1-\beta)\pi(s_t) \Vert \nonumber\\
	&\leq \beta\norm{\mathbb{E}_{\tilde{\A}}[\pih(\tilde{s}_t)] -\pih(s_t)} \nonumber\\ &\qquad+(1-\beta)\norm{\mathbb{E}_{\tilde{\A}}[\pi(\tilde{s}_t)]-\pi(s_t)} \nonumber \\
	&\leq \beta\norm{\mathbb{E}_{\tilde{\A}}[\tilde{a}_{t-1}]-a_{t-1}} \nonumber\\
	&\qquad+ (1-\beta)\norm{\mathbb{E}_{\tilde{\A}}[\tilde{a}_{t-1}]-a_{t-1}} \nonumber \\
	&=0 \nonumber
	\end{align*}
	per inductive hypothesis. Therefore we conclude that $\mathbb{E}_{\tilde{\A}}[\tilde{a}_t] = a_t \enskip\forall t = 1,\ldots,T$
\end{proof}
\subsection{Proof of theorem \ref{policy_improvement} and corollary \ref{cor:monotonic_improvement} - Main policy improvement results}
In this section, we provide the proof to theorem \ref{policy_improvement} and corollary \ref{cor:monotonic_improvement}. 
\begin{theorem_statement}{(theorem \ref{policy_improvement})}
	Assume $\ell$ is convex and $L$-Lipschitz-continuous, and Condition \ref{cond:self-smooth} holds. Let $\epsilon = \max\limits_{s\sim d_\pi}\norm{\pih(s) - \pi(s)}$.  Then for $\beta\in(0,1)$: 
	\begin{equation}
	\ell_{\pip}(\pip) - \ell_{\pi}(\pi) \leq \frac{\beta\gamma\epsilon L}{(1-\beta)(1-\gamma)} + \beta(\ell_{\pi}(\pih) - \ell_{\pi}(\pi)). \nonumber
	\end{equation}
\end{theorem_statement}
\begin{proof}

First let's review the notations: let $T$ be the trajectory horizon. For a policy $\pi$ in the deterministic policy class $\Pi$, given a starting state $s_0$, we roll out the full trajectory $s_0 \xrightarrow{\pi} s_1 \xrightarrow{\pi}\ldots \xrightarrow{\pi} s_T$, where $s_t = \left[x_t,\pi(s_{t-1})\right]$, with $x_t$ encodes the featurized input at current time $t$, and $\pi(s_{t-1})$ encodes the dependency on previous predictions. Let $\ell(\pi(s))$ be the loss of taking action $\pi(s)$ at state $s$, we can define the trajectory loss of policy $\pi$ from starting state $s_0$ as 
\begin{equation*}
\ell(\pi |s_0) = \frac{1}{T}\sum_{t=1}^T\ell(\pi(s_t))
\end{equation*}
For a starting state distribution $\mu$, we define policy loss of $\pi$ as the expected loss along trajectories induced by $\pi$: $\ell_\pi(\pi) = \mathbb{E}_{s_0\sim\mu}[ \ell(\pi | s_0 )]$.  Policy loss $\ell_\pi(\pi)$ can be understood as 
\begin{equation*}
\ell_\pi(\pi) = \int\limits_{s_0\sim\mu}\mathop{\mathbb{E}}\limits_{x_t\sim\mathcal{X}}\frac{1}{T}\left[ \sum_{t=1}^T \ell(\pi(s_t))\right]d\mu(s_0)
\end{equation*}
To prove policy improvement, we skip the subscript of algorithm \ref{algo:simile} to consider general policy update rule within each iteration: 
\begin{equation}
\pi^\prime = \pi_{new}= \beta \pih + (1-\beta)\pi
\label{update_rule}
\end{equation}
where $\pi = \pi_{old}$ is the current policy (combined up until the previous iteration), $\pih$ is the trained model from calling the base regression routine $\train(\mathbf{S}, \widehat{\A} | h)$. Learning rate (step-size) $\beta$ may be adaptively chosen in each iteration. Recall that this update rule reflects deterministic interpolation of two policies.

We are interested in quantifying the policy improvement when updating $\pi$ to $\pi^\prime$. Specifically, we want to bound 
\begin{equation*}
\Gamma = \ell_{\pip}(\pi^\prime) - \ell_{\pi}(\pi)
\end{equation*}
where $\ell_{\pi}(\pi)$ (respectively $\ell_{\pi^\prime}(\pi^\prime)$) denotes the trajectory loss of $\pi$ (respectively $\pi^\prime$) on the state distribution induced by $\pi$ (resp. $\pi^\prime$)


We will bound the loss difference of old and new policies conditioned on a common starting state $s_0$. Based on update rule \eqref{update_rule}, consider rolling out $\pi^\prime$ and $\pi$ from the same starting state $s_0$ to obtain two separate sequences $\pi^\prime \longmapsto \{ s_0 \rightarrow s_1^\prime \ldots \rightarrow s_T^\prime \}$ and $\pi \longmapsto \{ s_0 \rightarrow s_1 \ldots \rightarrow s_T \}$ corresponding to the same stream of inputs $x_1,\ldots,x_T$. 
\begin{align}
\Gamma(s_0) &= \frac{1}{T}\sum_{t=1}^T \ell(\pi^\prime(s_t^\prime)) - \ell(\pi(s_t)) \nonumber \\
&= \frac{1}{T}\sum_{t=1}^T \ell(\pi^\prime(s_t^\prime)) - \ell(\pip(s_t)) + \ell(\pip(s_t)) - \ell(\pi(s_t))
\label{eqn:delta_bound}
\end{align}
Assume convexity of $\ell$ (e.g. sum of square losses):
\begin{align*}
\ell(\pi^\prime(s_t)) &= \ell(\beta\pih(s_t) + (1-\beta)\pi(s_t)) \\
&\leq \beta\ell(\pih(s_t)) + (1-\beta)\ell(\pi(s_t))
\end{align*}
Thus we can begin to bound individual components of $\Gamma(s_0)$ as 
\begin{align*}
\ell(\pi^\prime(s_t^\prime)) - \ell(\pi(s_t)) &\leq \ell(\pip(s_t^\prime))) - \ell(\pip(s_t)) \\ &\quad +\beta \left[ \ell(\pih(s_t)) - \ell(\pi(s_t)) \right]
\end{align*}
Since $\ell$ is $L$-Lipschitz continuous, we have
\begin{align}
\ell(\pip(s_t^\prime))-\ell(\pip(s_t)) &\leq L\norm{\pip(s_t^\prime) - \pip(s_t)} \nonumber \\
&\leq L\gamma\norm{s_t^\prime - s_t}
\label{eqn:pi_prime_bound}
\end{align}
where (\ref*{eqn:pi_prime_bound}) is due to the smoothness condition [\ref{cond:self-smooth}] of policy class $\Pi$. Given a policy class $\Pi$ with $\gamma < 1$, the following claim can be proved by induction:
\\ \underline{\textbf{Claim:}} $\norm{s_t^\prime - s_t} \leq \frac{\beta\epsilon}{(1-\beta)(1-\gamma)}$  
\begin{proof}
For the base case, given the same starting state $s_0$, we have $s_1^\prime = \left[x_1,\pip(s_0)\right]$ and $s_1 = \left[x_1,\pi(s_0)\right]$. Thus $\norm{s_1^\prime - s_1} = \norm{\pip(s_0) - \pi(s_0)} = \norm{\beta\pih(s_0)+(1-\beta)\pi(s_0) - \pi(s_0)} = \beta\norm{\pih(s_0)-\pi(s_0)}\leq \beta\epsilon \leq \frac{\beta\epsilon}{(1-\beta)(1-\gamma)}$.

In the inductive case, assume we have $\norm{s_{t-1}^\prime - s_{t-1}} \leq \frac{\beta\epsilon}{(1-\beta)(1-\gamma)}$. Then similar to before, the definition of $s_t^\prime$ and $s_t$ leads to
\begin{align*}
\norm{s_t^\prime - s_t} &= \norm{\left[x_t,\pip(s_{t-1}^\prime)\right] - \left[x_t,\pi(s_{t-1})\right]} \\
&= \norm{\pip(s_{t-1}^\prime) - \pi(s_{t-1})} \\
&\leq \norm{\pip(s_{t-1}^\prime) - \pip(s_{t-1})} + \norm{\pip(s_{t-1}) - \pi(s_{t-1})} \\
&\leq \gamma\norm{s_{t-1}^\prime - s_{t-1}} + \beta\norm{\pih(s_{t-1})-\pi(s_{t-1})} \\
&\leq \gamma\frac{\beta\epsilon}{(1-\beta)(1-\gamma)}+\beta\epsilon \\
&\leq \frac{\beta\epsilon}{(1-\beta)(1-\gamma)}
\end{align*}
\end{proof}
Applying the claim to equation (\ref{eqn:pi_prime_bound}), we have
\begin{equation}
\ell(\pip(s_t^\prime)) - \ell(\pip(s_t)) \leq \frac{\beta\gamma\epsilon L}{(1-\beta)(1-\gamma)} \nonumber
\end{equation}
which leads to
\begin{align}
\ell(\pip(s_t^\prime) - \ell(\pi(s_t))) &\leq \frac{\beta\gamma\epsilon L}{(1-\beta)(1-\gamma)} \nonumber \\ 
&\quad+ \beta(\ell(\pih(s_t)) -\ell(\pi(s_t)))
\label{eqn:one_state_bound}
\end{align}
Integrating \eqref{eqn:one_state_bound} over the starting state $s_0\sim\mu$ and input trajectories $\{x_t\}_{t=1}^T$, we arrive at the policy improvement bound:
\begin{equation}
\ell_{\pip}(\pip) - \ell_{\pi}(\pi) \leq \frac{\beta\gamma\epsilon L}{(1-\beta)(1-\gamma)} + \beta(\ell_{\pi}(\pih) - \ell_{\pi}(\pi)) \nonumber
\end{equation}
where $\ell_{\pi}(\pih)$ is the expected loss of the trained policy $\pih$ on the state distribution induced by policy $\pi$ (reduction term, analogous to policy advantage in the traditional MDP terminologies \cite{kakade2002approximately})
\end{proof}
This means in the worst case, as we choose $\beta\rightarrow 0$, we have $\left[\ell_{\pip}(\pi^\prime) - \ell_{\pi}(\pi)\right] \rightarrow 0$, meaning the new policy does not degrade much for a small choice of $\beta$. However if $\ell_\pi(\pih) - \ell_\pi(\pi) \ll 0$, we can choose $\beta$ to enforce monotonic improvement of the policy by adaptively choosing $\beta$ that minimizes the right-hand side. In particular, let the reduction term be $\Delta = \ell_\pi(\pi) - \ell_{\pi}(\pih) > 0$ and let $\delta = \frac{\gamma\epsilon L}{1-\gamma}$, then for $\beta =\frac{\Delta-\delta}{2\Delta}$we have the following monotonic policy improvement:
\begin{equation*}
\ell_{\pip}(\pip) - \ell_{\pi}(\pi) \leq -\frac{(\Delta-\delta)^2}{2(\Delta+\delta)}
\end{equation*}
\subsection{Proof of theorem \ref{first_theorem} - $T$-dependent improvement}
\begin{theorem_statement}{(theorem \ref{first_theorem})}
	Assume $\ell$ is convex and $L$-Lipschitz, and Condition \ref{cond:stability_condition} holds. Let $\epsilon = \max\limits_{s\sim d_\pi}\norm{\pih(s) - \pi(s)}$. Then: 
	\begin{equation}
	\ell_{\pi^\prime}(\pi^\prime) - \ell_{\pi}(\pi) \leq \beta\epsilon LT + \beta\left(\ell_{\pi}(\pih) - \ell_{\pi}(\pi)\right). \nonumber
	\end{equation}
	In particular, choosing $\beta\in(0,1/T)$ yields:
	\begin{equation}
	\ell_{\pi^\prime}(\pi^\prime) - \ell_{\pi}(\pi) \leq \epsilon L + \beta\left(\ell_{\pi}(\pih) - \ell_{\pi}(\pi)\right). \nonumber
	\end{equation}
\end{theorem_statement}
\begin{proof}
	The proof of theorem \ref{first_theorem} largely follows the structute of theorem \ref{policy_improvement}, except that we are using the slighty weaker Condition \ref{cond:stability_condition} which leads to weaker bound on the policy improvement that depends on the trajectory horizon $T$. 
	For any state $s_0$ taken from the starting state distribution $\mu$, sequentially roll-out policies $\pip$ and $\pi$ to receive two separate trajectories $\pip: s_0\rightarrow s^\prime_1 \rightarrow \ldots \rightarrow s^\prime_T$ and $\pip: s_0\rightarrow s_1 \rightarrow \ldots \rightarrow s_T$. Consider a pair of states $s^\prime_t = [x_t, \pip(s^\prime_{t-1})]$ and $s_t = [x_t, \pi(s_{t-1})]$ corresponding to the same input feature $x_t$, as before we can decompose $\ell(\pip(s^\prime_t)) - \ell(\pi(s_t)) = \ell(\pip(s^\prime_t)) - \ell(\pip(s_t)) + \ell(\pip(s_t)) -\ell(\pi(s_t)) \leq L\norm{\pip(s^\prime_t) - \pip(s_t)} + \beta(\ell(\pih(s_t)) - \ell(\pi(s_t)))$ due to convexity and $L$-Lipschitz continuity of $\ell$.
	
	Condition \ref{cond:stability_condition} further yields:
	$\ell(\pip(s^\prime_t)) - \ell(\pi(s_t)) \leq L\norm{s^\prime_t - s_t} + \beta(\ell(\pih(s_t)) - \ell(\pi(s_t)))$.
	By the construction of the states, note that 
	\begin{align*}
\norm{s^\prime_t - s_t} &= \norm{\pip(s^\prime_{t-1}) - \pi(s_{t-1})} \\ 
&\leq \norm{\pip(s^\prime_{t-1})-\pip(s_{t-1})} + \norm{\pip(s_{t-1}) - \pi(s_{t-1})} \\ 
&\leq \norm{s^\prime_{t-1} - s_{t-1}}+\beta(\norm{\pih(s_{t-1})-\pi(s_{t-1})}) \\ 
&\leq \norm{s^\prime_{t-1} - s_{t-1}} + \beta\epsilon
	\end{align*}
	(by condition \ref{cond:stability_condition} and definition of $\epsilon$).
	
	From here, one can use this recursive relation to easily show that $\norm{s^\prime_t - s_t} \leq \beta\epsilon t$ for all $t \in [1,T]$.
	
	Averaging over the $T$ time steps and integrating over the starting state distribution, we have:
	\begin{align*}
\ell_{\pi^\prime}(\pi^\prime) - \ell_{\pi}(\pi) &\leq \beta\epsilon L(T+1)/2 + \beta(\ell_{\pi}(\pih) - \ell_{\pi}(\pi)) \\
&\leq \beta\epsilon LT + \beta(\ell_{\pi}(\pih) - \ell_{\pi}(\pi))
	\end{align*} 
	In particular, $\beta\in(0,1/T)$ yields $\ell_{\pi^\prime}(\pi^\prime) - \ell_{\pi}(\pi) \leq \epsilon L+\beta(\ell_{\pi}(\pih) - \ell_{\pi}(\pi))$.
\end{proof}
\subsection{Proof of proposition \ref{prop:smooth_expert} - smooth expert proposition }
\begin{prop_statement}{(Proposition \ref{prop:smooth_expert})}
	Let $\omega$ be the average supervised training error from $\mathcal{F}$, i.e. $\omega = \min\limits_{f\in\mathcal{F}}\mathbb{E}_{x\sim\mathcal{X}}\left[\norm{f([x,0])-a^*}\right]$. Let the rolled-out trajectory of current policy $\pi$ be $\{a_t\}$. If the average gap between $\pi$ and $\pis$ is such that $\mathbb{E}_{t\sim \text{Uniform}[1:T]}\left[\norm{a_t^*-a_{t-1}}\right] \geq 3\omega+\eta(1+\lambda)$, then using $\{a_t^*\}$ as feedback will cause the trained policy $\pih$ to be non-smooth, i.e.:
	\begin{equation}
	\mathbb{E}_{t\sim \text{Uniform}[1:T]}\left[\norm{\hat{a}_t-\hat{a}_{t-1}}\right] \geq \eta, \nonumber
	\end{equation}
	for $\{\hat{a}_t\}$ the rolled-out trajectory of $\pih$.
\end{prop_statement}
\begin{proof}
Recall that $\Pi_{\lambda}$ is formed by regularizing a class of supervised learners $\mathcal{F}$ with the singleton class of smooth function $\mathcal{H} \triangleq \{ h(a) = a\}$, via a hyper-parameter $\lambda$ that controls the trade-off between being close to the two classes. 

Minimizing over $\Pi_{\lambda}$ can be seen as a regularized optimization problem:
\begin{align}
\pih(x,a) &= \argmin_{\pi\in\Pi}\ell(\pi(\left[x,a\right])) \nonumber
\\&= \argmin_{f\in\mathcal{F}, h\in\mathcal{H}}(f(x,a)-a^*)^2+\lambda(f(x,a)-h(a))^2 \nonumber \\
&= \argmin_{f\in\mathcal{F}}(f(x,a)-a^*)^2+\lambda(f(x,a)-a)^2 \label{eqn:f_objective}
\end{align}
where hyper-parameter $\lambda$ trades-off the distance of $f(x,a)$ relative to $a$ (smoothness) and $a^*$ (imitation accuracy), and $a\in\mathbb{R}^1$. 

Such a policy $\pi$, at execution time, corresponds to the regularized minimizer of:
\begin{align}
a_{t} &= \pi(\left[x,a_{t-1}\right]) \nonumber \\
&= \argmin\limits_{a} \norm{a-f(\left[x_t,a_{t-1}\right])}^2 +\lambda\norm{a-a_{t-1}}^2 \nonumber \\
&= \frac{f(\left[x_t,a_{t-1}\right])+\lambda a_{t-1}}{1+\lambda} \label{eqn:a_objective}
\end{align}
where $f\in\mathcal{F}$ is the minimizer of equation \ref{eqn:f_objective}

Thus we enforce smoothness of learning policy from $\Pi_{\lambda}$ by encouraging low first order difference of consecutive actions of the executed trajectory $\{a_t\}$. In practice, we may contrain this first order difference relative to the human trajectory $\frac{1}{T}\sum_{t=1}^T \norm{a_t - a_{t-1}} \leq \eta$, where $\eta\propto \frac{1}{T}\sum_{t=1}^T\norm{a_t^* - a_{t-1}^*}$.  

Consider any given iteration with the following set-up:  we execute old policy $\pi=\pi_{old}$ to get rolled-out trajectory $\{a_t\}_{t=1}^T$. Form the new data set as $\mathcal{D} = \{(s_t,a_t^*)\}_{t=1}^T$ with predictors $s_t = \left[x_t,a_{t-1}\right]$ and feedback labels simply the human actions $a_t^*$. Use this data set to train a policy $\pih$ by learning a supervised $\hat{f}\in\mathcal{F}$ from $\mathcal{D}$. Similar to $\pi$, the execution of $\pih$ corresponds to $\hat{a}_t$ where:
\begin{align}
\hat{a}_t &= \pih(\left[x_t,\hat{a}_{t-1}\right]) \nonumber \\
&= \argmin\limits_{a} \norm{a-\hat{f}(\left[x_t,\hat{a}_{t-1}\right])}^2 +\lambda\norm{a-\hat{a}_{t-1}}^2 \nonumber \\
&= \frac{\hat{f}(\left[x_t,\hat{a}_{t-1}\right])+\lambda \hat{a}_{t-1}}{1+\lambda} \label{eqn:ahat_objective}
\end{align}
 
 Denote by $f_{0}$ the "naive" supervised learner from $\mathcal{F}$. In other words, $f_0 = \argmin\limits_{f\in\mathcal{F}}\sum\limits_{t=1}^T \norm{f(\left[x_t,0\right]) - a_t^*}^2$. Let $\omega$ be the average gap between human trajectory and the rolled-out trajectory of $f_0$, i.e.
 \begin{equation*}
 \omega = \frac{1}{T}\sum_{t=1}^T \norm{f_0(\left[x_t,0\right]) - a_t^*}
 \end{equation*}
 Note that it is reasonable to assume that the average errors of $f$ and $\hat{f}$ are no worse than $f_0$, since in the worst case we can simply discard the extra features $a_{t-1}$ (resp. $\hat{a}_{t-1}$) of $f$ (resp. $\hat{f}$) to recover the performance of the naive learner $f_0$:
 \begin{align*}
\frac{1}{T}\sum_{t=1}^T\norm{f([x_t,a_{t-1}]) - a_t^*} &\leq \omega \\
\frac{1}{T}\sum_{t=1}^T\norm{\hat{f}([x_t,\hat{a}_{t-1}]) - a_t^*} &\leq \omega
 \end{align*}
 
Assume that the old policy $\pi = \pi_{old}$ is "bad" in the sense that the rolled-out trajectory $\{a_t\}_{t=1}^T$ differs substantially from human trajectory $\{a_t^*\}_{t=1}^T$. Specifically, denote the gap:
\begin{equation*}
\frac{1}{T}\sum_{t=1}^T\norm{a_t^*-a_{t-1}}=\Omega \gg \omega
\end{equation*}
This means the feedback correction $a_t^*$ to $s_t=\left[x_t,a_{t-1}\right]$ is not smooth. We will show that the trained policy $\pih$ from $\mathcal{D}$ will not be smooth. 

From the definition of $a_t$ and $\hat{a}_t$ from equations \ref{eqn:a_objective} and \ref{eqn:ahat_objective}, we have for each $t$:
\begin{equation}
a_t - \hat{a}_{t} = \frac{\lambda}{1+\lambda}(a_{t-1} - \hat{a}_{t-1}) + \frac{f([x_t,a_{t-1}])-\hat{f}([x_t,\hat{a}_{t-1}])}{1+\lambda} \nonumber
\end{equation}
Applying triangle inequality and summing up over $t$, we have:
\begin{equation}
\frac{1}{T}\sum_{t=1}^{T}\norm{a_t-\hat{a}_{t}} \leq 2\omega \nonumber
\end{equation}
From here we can provide a lower bound on the smoothness of the new trajectory $\hat{a}_t$, as defined by the first order difference $\frac{1}{T}\sum_{t=1}^{T}\norm{\hat{a}_t-\hat{a}_{t-1}}$. By definition of $\hat{a}_t$:
\begin{align}
&\norm{\hat{a}_t - \hat{a}_{t-1}} = \norm{\frac{\hat{f}([x_t,\hat{a}_{t-1}]) - \hat{a}_{t-1}}{1+\lambda} } \nonumber \\
&= \norm{\frac{\hat{f}([x_t,\hat{a}_{t-1}]) - a_t^* + a_t^*-a_{t-1}+a_{t-1}- \hat{a}_{t-1}}{1+\lambda} } \nonumber \\
&\geq \frac{\norm{a_t^* - a_{t-1}} -\norm{\hat{f}([x_t,\hat{a}_{t-1}])-a_t^*}  - \norm{a_{t-1} -\hat{a}_{t-1}} }{1+\lambda} \nonumber
\end{align}
Again summing up over $t$ and taking the average, we obtain:
\begin{equation}
\frac{1}{T}\sum_{t=1}^{T}\norm{\hat{a}_t - \hat{a}_{t-1}} \geq \frac{\Omega - 3\omega}{1+\lambda} \nonumber
\end{equation}
Hence for $\Omega\gg\omega$, meaning the old trajectory is sufficiently far away from the ideal human trajectory, setting the learning target to be the ideal human actions will cause the learned trajectory to be non-smooth. 
\end{proof}
\section{Imitation Learning for Online Sequence Prediction With Smooth Regression Forests}
\label{sec:tree}
\subsection{Variant of SIMILE Using Smooth Regression Forest Policy Class}
We provide a specific instantiation of algorithm \ref{algo:simile} that we used for our experiment, based on a policy class $\Pi$ as a smooth regularized version of the space of tree-based ensembles. In particular, $\mathcal{F}$ is the space of random forests and $\mathcal{H}$ is the space of linear auto-regressors $\mathcal{H} \triangleq\{h(a_{t-1:t-\tau})=\sum_{i=1}^\tau c_i a_{t-i} \}$. In combination, $\mathcal{F}$ and $\mathcal{H}$ form a complex tree-based predictor that can predict smooth sequential actions. 

Empirically, decision tree-based ensembles are among the best performing supervised machine learning method \cite{caruana2006empirical,DecisionForests}. Due to the piece-wise constant nature of decision tree-based prediction, the results are inevitably non-smooth. We propose a recurrent extension based on $\mathcal{H}$, where the prediction at the leaf node is not necessarily a constant, but rather is a smooth function of both static leaf node prediction and its previous predictions. By merging the powerful tree-based policy class with a linear auto-regressor, we provide a novel approach to train complex models that can accommodate smooth sequential prediction using model-based smooth regularizer, at the same time leveraging the expressiveness of complex model-free function class (one can similarly apply the framework to the space of  neural networks).
\begin{algorithm}[tb]
	\caption{ Imitation Learning for Online Sequence Prediction with Smooth Regression Forest}
	\label{algo:simile_tree}
	\begin{algorithmic}[1]
		\REQUIRE Input features $\mathbf{X} = \{x_t\}_{t=1}^T$, expert demonstration $\A^* = \{a_t^*\}_{t=1}^T$, base routine $\texttt{Forest}$, past horizon $\tau$, sequence of $\sigma\in(0,1)$ \\
		\STATE Initialize $\A_0 \leftarrow \A^*, \s_0\leftarrow \{\left[x_{t:t-\tau},a_{t-1:t-\tau}^*\right]\} $, \\ $\qquad\qquad h_0 = \argmin\limits_{c_1,\ldots,c_\tau}\sum\limits_{t=1}^T\left(a_t^*-\sum_{i=1}^\tau c_i a_{t-i}^*\right)^2$ \\
		\STATE Initial policy $\pi_0 = \pih_0\leftarrow $\texttt{Forest}$(\mathbf{S}_0,\mathbf{A}_0|\enskip h_0)$ \\
		\FOR{$n = 1,\ldots, N$}
		\STATE $\mathbf{A}_n  = \{a_t^n\} \leftarrow \{\pi_{n-1}(\left[x_{t:t-\tau},a_{t-1:t-\tau}^{n-1}\right])\}$ \\ \hfill{//\textit{sequential roll-out old policy} } \\
		\STATE $\s_n \leftarrow \{s_t^n = \left[x_{t:t-\tau},a_{t-1:t-\tau}^n\right]\} $ \hfill{//\textit{Form states}} \\ \hfill{\textit{in 1d case} } \\
		\STATE $\widehat{\A}_n = \{\widehat{a}_t^n =\sigma a_t^n+(1-\sigma)a_t^*\} \enskip \forall s_t^n\in\mathbf{S}_n$ \\ \hfill{// \textit{collect smooth 1-step feedback}} \\
		\STATE $h_n = \argmin\limits_{c_1,\ldots,c_\tau}\sum\limits_{t=1}^T\left(\hat{a}_t^n-\sum_{i=1}^\tau c_i \hat{a}_{t-i}^n\right)^2$ \hfill{//\textit{update $c_i$}} \\ \hfill{\textit{ via regularized least square}}  \\
		\STATE $\pih_n \leftarrow $\texttt{Forest}$(\mathbf{S}_n, \widehat{\mathbf{A}}_n |\enskip h_n)$ \label{algo:learn2} \hfill{// \textit{train with smooth}} \\ \hfill{\textit{decision forests. See section \ref{sec:smooth_tree}}} \\
		\STATE $\beta \leftarrow \frac{\texttt{error}(\pi)}{\texttt{error}(\pih) + \texttt{error}(\pi)}$ \hfill{//\textit{set $\beta$ to weighted }} \\ \hfill{\textit{empirical errors }} \\
		\STATE $\pi_n = \beta\pih_n + (1-\beta)\pi_{n-1}$ \label{algo:interpolate2} \hfill{// \textit{update policy }} \\
		\ENDFOR \\
		\OUTPUT Last policy $\pi_N$
	\end{algorithmic}
\end{algorithm}
Algorithm \ref{algo:simile_tree}, which is based on SIMILE, describes in more details our training procedure used for the automated camera planning experiment. We first describe the role of the linear autoregressor class $\mathcal{H}$, before discussing how to incorporate $\mathcal{H}$ into decision tree training to make smooth prediction (see the next section). 

The autoregresor $h_\pi(a_{-1},\ldots,a_{-\tau})$ is typically selected from a class of autoregressors $\mathcal{H}$. In our experiments, we use regularized linear autoregressors as $\mathcal{H}$. 

Consider a generic learning policy $\pi$ with a rolled-out trajectory $\A = \{ a_t\}_{t=1}^T$ corresponding to the input sequence $\X = \{ x_t\}_{t=1}^T$.  We form the state sequence $\mathbf{S} = \{ s_t \}_{t=1}^T = \{\left[ x_t,\ldots,x_{t-\tau}, a_{t-1}, \ldots, a_{t-\tau} \right] \}_{t=1}^T$, where $\tau$ indicates the past horizon that is adequate to approximately capture the full state information. We approximate the smoothness of the trajectory $\A$  by a linear autoregressor 
\begin{equation*}
h_\pi \equiv h_\pi(s_t) \equiv \sum_{i=1}^\tau c_i a_{t-i}
\end{equation*}
for a (learned) set of coefficients $\{ c_i \}_{i=1}^\tau$ such that $a_t \approx h_\pi\left( s_t \right)$. Given feedback target $\widehat{\A} = \{ \hat{a}_t\}$, the joint loss function thus becomes
\begin{align*}
\ell(a,\hat{a}_t) &= \ell_d(a,\hat{a}_t) + \lambda \ell_R(a, s_t) 
\\ &= (a-\hat{a}_t)^2+\lambda (a-\sum_{i=1}^\tau c_i a_{t-i})^2
\end{align*} 
Here $\lambda$ trades off between smoothness versus absolute imitation accuracy. The autoregressor $h_\pi$ acts as a smooth linear regularizer, the parameters of which can be updated at each iteration based on feedback target $\widehat{\A}$ according to 
\begin{align}
h_{\pi} &= \argmin_{h\in \mathcal{H}} \norm{\widehat{\A} - h(\widehat{\A})}^2 \nonumber 
\\&= \argmin_{c_1,\ldots, c_\tau}(\sum_{t=1}^T (\hat{a}_t-\sum_{i=1}^\tau c_i \hat{a}_{t-i})^2), \label{eqn: update_c}
\end{align}
In practice we use a regularized version of equation \eqref{eqn: update_c} to learn a new set of coefficients $\{c_i \}_{i=1}^\tau$. The $\texttt{Forest}$ procedure (Line 8 of algorithm 2) would use this updated $h_\pi$ to train a new policy that optimizes the trade-off between $a_t \approx \hat{a}_t$ (feedback) versus smoothness as dictated by $a_t \approx \sum_{i=1}^\tau c_i a_{t-i}$. 

\subsubsection{Smooth Regularization with Linear Autoregressors}
Our application of Algorithm 1 to realtime camera planning proceeds as follows: At each iteration, we form a state sequence $\s$ based on the rolled-out trajectory $\A$ and tracking input data $\X$ such that $s_t = \left[x_t,\ldots,x_{t-\tau},a_{t-1},\ldots,a_{t-\tau}\right]$ for appropriate $\tau$ that captures the history of the sequential decisions.  We generate feedback targets $\widehat{\A}$ based on each $s_t\in\s$ following $\hat{a}_t = \sigma a_t + (1-\sigma)a_t^*$ using a parameter $\sigma\in(0,1)$ depending on the Euclidean distance between $\A$ and $\A^*$. Typically, $\sigma$ gradually decreases to $0$ as the rolled-out trajectory improves on the training set. After generating the targets, a new linear autoregressor $h_\pi$ (new set of coefficients $\{c_i\}_{i=1}^\tau$) is learned based on $\widehat{\A}$ using regularized least squares (as described in the previous section). We then train a new model $\pih$ based on $\s, \widehat{\A}$, and the updated coefficients $\{c_i\}$, using $\texttt{Forest}$ - our recurrent decision tree framework that is capable of generating smooth predictions using autoregressor $h_{\pi}$ as a smooth regularizer (see the following section for how to train smooth decision trees). Note that typically this creates a "chicken-and-egg" problem. As the newly learned policy $\pih$ is greedily trained with respect to $\widehat{\A}$, the rolled-out trajectory of $\pih$ may have a state distribution that is different from what the previously learned $h_\pi$ would predict. Our approach offers two remedies to this circular problem. First, by allowing feedback signals to vary smoothly relative to the current rolled-out trajectory $\A$, the new policy $\pih$ should induce a new autoregresor that is similar to previously learned $h_\pi$. Second, by interpolating distributions (Line 10 of Algorithm 2) and having $\widehat{\A}$ eventually converge to the original human trajectory $\A^*$, we will have a stable and converging state distribution, leading to a stable and converging $h_\pi$. 

Throughout iterations, the linear autoregressor $h_\pi$ and regularization parameter $\lambda$ enforces smoothness of the rolled-out trajectory, while the recurrent decision tree framework $\texttt{Forest}$ learns increasingly accurate imitation policy. We generally achieve a satisfactory policy after 5-10 iterations in our sport broadcasting data sets. In the following section, we describe the mechanics of our recurrent decision tree training. 
\subsection{Smooth Regression Tree Training}
\label{sec:smooth_tree}
Given states $s$ as input, a decision tree specifies a partitioning of the input state space.  Let $D = \{(s_m,\hat{a}_m)\}_{m=1}^M$ denote a training set of state/target pairs.  Conventional regression tree learning aims to learn a partitioning such that each leaf node, $\node$, makes a constant prediction via minimizing the squared loss function:
\begin{align}
	\bar{a}_{\node} &= \argmin_{a} \sum_{(s,\hat{a}) \in D_\node} \ell_{d}(a,\hat{a}) \nonumber
	\\&= \argmin_{a} \sum_{(s,\hat{a}) \in D_{\node}}(\hat{a} - a)^2,\label{eqn:tree1}
\end{align}
where $D_\node$ denotes the training data from $D$ that has partitioned into the leaf $\node$.  For squared loss, we have:
\begin{align}\bar{a}_{\node} = \mean\left\{\hat{a}\left|(s,\hat{a}) \in D_{\node}\right.\right\}.\label{eqn:tree1_predict}\end{align}

In the recurrent extension to $\texttt{Forest}$, we allow the decision tree to branch on the input state $s$, which includes the previous predictions $a_{-1},\ldots,a_{-\tau}$. To enforce more explicit smoothness requirements, let $h_\pi(a_{-1},\ldots,a_{-\tau})$ denote an  autoregressor that captures the temporal dynamics of $\pi$ over the distribution of input sequences $d_\x$, while \textit{ignoring} the inputs $x$. At time step $t$, $h_\pi$ predicts  the behavior $a_t = \pi(s_t)$ given only $a_{t-1},\ldots,a_{t-\tau}$.  

Our policy class $\Pi$ of recurrent decision trees $\pi$  makes smoothed predictions by regularizing the predictions to be close to its autoregressor $h_\pi$. The new loss function incorporates both the squared distance loss $\ell_d$, as well as a smooth regularization loss such that: 
\begin{align*}
	\mathcal{L}_D(a) &= \sum_{(s,\hat{a}) \in D} \ell_{d}(a,\hat{a}) + \lambda\ell_{R}(a,s) 
	\\ &= \sum_{(s,\hat{a}) \in D} (a-\hat{a})^2 + \lambda(y-h_{\pi}(s))^2
\end{align*}
where $\lambda$ is a hyper-parameter that controls how much we care about smoothness versus absolute distance loss. 

\textbf{Making prediction:} For any any tree/policy $\pi$, each leaf node is associated with the terminal leaf node value $\bar{a}_\node$ such that prediction $\tilde{a}$ given input state $s$ is:
\begin{align}
	\tilde{a}(s) \equiv \pi(s) &= \argmin_a~(a- \bar{a}_{\node(s)})^2 + \lambda(a - h_\pi(s))^2,\label{eqn:tree2_predict}
	\\ &= \frac{\bar{a}_{\node(s)} + \lambda h_\pi(s)}{1+\lambda}.\label{eqn:tree2_predictb}
\end{align}
where $\node(s)$ denotes the leaf node of the decision tree that $s$ branches to.

\textbf{Setting terminal node value:} 
Given a fixed $h_\pi$ and decision tree structure, navigating through consecutive binary queries eventually yields a terminal leaf node with associated training data $D_\node \subset D$. 

One option is to set the terminal node value $\bar{a}_\node$ to satisfy:
\begin{align}
	\bar{a}_{\node} &= \argmin_{a}\sum_{(s,\hat{a}) \in D_\node} \ell_d(\tilde{a}(s|a), \hat{a}) \nonumber
	\\ &= \argmin_{a} \sum_{(s,\hat{a}) \in D_\node}(\tilde{a}(s|a)-\hat{a})^2
	\label{eqn:tree2}
	\\ &= \argmin_{a} \sum_{(s,\hat{a}) \in D_\node} \left(\frac{a+\lambda h_\pi(s)}{1+\lambda}-\hat{a}\right)^2 \nonumber
\end{align}
for $\tilde{a}(s|a)$ defined as in \eqref{eqn:tree2_predictb} with $a\equiv \bar{a}_{\node(s)}$.  Similar to \eqref{eqn:tree1_predict}, we can write the closed-form solution of \eqref{eqn:tree2} as:
\begin{eqnarray}
	\begin{small}
		\bar{a}_{\node} = \texttt{mean}\left\{(1+\lambda)\hat{a} - \lambda h_\pi(s) \left| (s,\hat{a}) \in  D_{\node}\right.\right\}.
	\end{small}
	\label{eqn:tree2_learn}
\end{eqnarray}
When $\lambda=0$,  \eqref{eqn:tree2_learn} reduces to \eqref{eqn:tree1_predict}.

Note that \eqref{eqn:tree2} only looks at imitation loss $\ell_d$, but not smoothness loss $\ell_R$.  Alternatively in the case of joint imitation and smoothness loss, the terminal leaf node is set to minimize the joint loss function:

\begin{align}
	&\bar{a}_{\node} = \argmin_{a} \mathcal{L}_{D_{\node}}(\tilde{a}(s|a)) \nonumber
	\\ &= \argmin_{a}\sum_{(s,\hat{a}) \in D_\node} \ell_d(\tilde{a}(s|a), \hat{a}) + \lambda\ell_R(\tilde{a}(s|a), s) \nonumber
	\\ &= \argmin_{a} \sum_{(s,\hat{a}) \in D_\node}(\tilde{a}(s|a)-\hat{a})^2 + \lambda (\tilde{a}(s|a)-h_\pi(s))^2
	\label{eqn:tree3}
	\\ &=\argmin_{a}\sum_{(s,\hat{a}) \in D_\node} \left(\frac{a + \lambda h_{\pi}(s)}{1+\lambda}-\hat{a}\right)^2 \nonumber
	\\ &+\lambda\left(\frac{a+\lambda h_\pi(s)}{1+\lambda}-h_\pi(s)\right)^2 \nonumber
	\\ &= \texttt{mean}\left\{\hat{a}\left| (s,\hat{a}) \in  D_{\node}\right.\right\}, \label{eqn:tree3_learn}
\end{align}

\textbf{Node splitting mechanism:} 
For a node representing a subset $D_\node$ of the training data, the node impurity is defined as:
\begin{align*} 
	I_{\node} &= \mathcal{L}_{D_\node}(\bar{a}_\node) 
	\\ &= \sum_{(s,\hat{a}) \in D_\node} \ell_d(\bar{a}_\node, \hat{a}) + \lambda\ell_R(\bar{a}_\node, s)
	\\ &= \sum_{(s,\hat{a}) \in D_\node} (\bar{a}_\node - \hat{a})^2 + \lambda(\bar{a}_\node-h_\pi(s))^2
\end{align*}
where $\bar{a}_\node$ is set according to equation \eqref{eqn:tree2_learn} or \eqref{eqn:tree3_learn} over $(s,\hat{a})$'s in $D_\node$. At each possible splitting point where $D_\node$ is partitioned into $D_{\texttt{left}}$ and $D_{\texttt{right}}$, the impurity of the left and right child of the node is defined similarly. As with normal decision trees, the best splitting point is chosen as one that maximizes the impurity reduction: $I_{\node} - \frac{\abs{D_{\texttt{left}}}}{\abs{D_\node}}I_{\texttt{left}} - \frac{\abs{D_{\texttt{right}}}}{\abs{D_\node}}I_{\texttt{right}} $

\end{document}